\documentclass[11pt]{article}
\pdfoutput=1

\usepackage{amsmath,amssymb}
\usepackage{bm}
\usepackage{natbib}
\usepackage[usenames]{color}
\usepackage{amsthm}

\usepackage{multirow} 
\usepackage{enumerate}
\usepackage{bbm}

\usepackage{svg}
\allowdisplaybreaks

\usepackage[colorlinks,
linkcolor=red,
anchorcolor=blue,
citecolor=blue
]{hyperref}

\graphicspath{{arxiv_v3/}}

\usepackage{setspace}
\usepackage[left=1in, right=1in, top=1in, bottom=1in]{geometry}

\usepackage{xcolor}
\usepackage{footnote}

\usepackage[utf8]{inputenc} 
\usepackage[T1]{fontenc}    
\usepackage{hyperref}       
\usepackage{url}            
\usepackage{booktabs}       
\usepackage{multirow}
\usepackage{amsfonts}       
\usepackage{nicefrac}       
\usepackage{microtype}      

\usepackage{graphicx}
\usepackage{amsthm}
\usepackage{amsmath}
\usepackage{amssymb}
\usepackage{enumerate} 
\usepackage{mathrsfs} 
\usepackage{listings} 
\usepackage{color} 
\definecolor{mygreen}{RGB}{28,172,0} 
\definecolor{mylilas}{RGB}{170,55,241}
\usepackage{caption}
\usepackage{subcaption} 
\usepackage{yfonts} 
\usepackage{bbm} 

\theoremstyle{definition}
\newtheorem{theorem}{Theorem}[section]
\newtheorem{corollary}[theorem]{Corollary}
\newtheorem{lemma}[theorem]{Lemma}

\newtheorem{definition}[theorem]{Definition}

\newtheorem{assumption}[theorem]{Assumption}
\theoremstyle{definition}
\newtheorem{example}[theorem]{Example}

\newcommand{\N}{\mathbb{N}}

\newcommand{\R}{\mathbb{R}}

\newcommand{\E}{\mathbb{E}}

\let\P\BP

\let\hat\widehat

\newcommand{\f}{\frac}

\newcommand{\eps}{\varepsilon}
\renewcommand{\r}{\right}
\renewcommand{\l}{\left}

\newcommand{\pnorm}[2]{\left\|#1\right\|_{#2}}
\newcommand{\norm}[1]{\left\|#1\right\|}
\newcommand{\snorm}[1]{\|#1\|}
\newcommand{\sip}[2]{\langle #1, #2 \rangle}
\newcommand{\ip}[2]{\left\langle #1, #2 \right\rangle}

\newcommand{\ind}{{\mathbbm{1}}}

\newcommand{\summ}[2]{\sum_{#1 = 1}^{#2}}
\newcommand{\summm}[3]{\sum_{#1 = #2}^{#3}}

\newcommand{\sgn}{\operatorname{sgn}}

\newcommand{\calD}{\mathcal{D}}
\newcommand{\calE}{\mathcal{E}}

\DeclareMathOperator{\poly}{poly}

\newcommand{\iid}{\stackrel{\rm i.i.d.}{\sim}}

\makeatletter
\renewcommand*{\eqref}[1]{%
  \hyperref[{#1}]{\textup{\tagform@{\ref*{#1}}}}%
}
\makeatother

\newcommand{\opt}{\mathsf{OPT}}

\newcommand{\esgd}{\mathbb{E}_{\mathrm{sgd}}}

\newcommand{\wt}[1]{w^{(#1)}}

\newcommand{\Wt}[1]{W^{(#1)}}

\numberwithin{equation}{section}
\DeclareMathAlphabet\mathbfcal{OMS}{cmsy}{b}{n}
\makeatletter
\@tfor\next:=abcdefghijklmnopqrstuvwxyzABCDEFGHIJKLMNOPQRSTUVWXYZ\do{%
	\def\command@factory#1{%
		\expandafter\def\csname b#1\endcsname{\mathbf{#1}}
		\expandafter\def\csname bb#1\endcsname{\mathbb{#1}}
		\expandafter\def\csname cl#1\endcsname{\mathcal{#1}}
		\expandafter\def\csname bcl#1\endcsname{\mathbfcal{#1}}
	}
	\expandafter\command@factory\next
}

\newcommand{\err}{\mathrm{err}}

\newcommand{\optlin}{\mathsf{OPT}_{\mathsf{lin}}}

\ifdefined\final
\usepackage[disable]{todonotes}
\else
\usepackage[textsize=tiny]{todonotes}
\fi
\author
{
    Spencer Frei\thanks{Department of Statistics, University of California, Los Angeles, CA 90095, USA; e-mail: {\tt spencerfrei@ucla.edu}}
    ~~~and~~~
	Yuan Cao\thanks{Department of Computer Science, University of California, Los Angeles, CA 90095, USA; e-mail: {\tt yuancao@cs.ucla.edu}} 
	~~~and~~~
	Quanquan Gu\thanks{Department of Computer Science, University of California, Los Angeles, CA 90095, USA; e-mail: {\tt qgu@cs.ucla.edu}}
}
\date{}
\title{\huge Provable Generalization of SGD-trained Neural Networks of Any Width in the Presence of Adversarial Label Noise}

\begin{document}
\maketitle

\begin{abstract}
    We consider a one-hidden-layer leaky ReLU network of arbitrary width trained by stochastic gradient descent (SGD) following an arbitrary initialization.   We prove that SGD  produces neural networks that have classification accuracy competitive with that of the best halfspace 
    over the distribution for a broad class of distributions that includes log-concave isotropic and hard margin distributions.  Equivalently, such networks can generalize when the data distribution is linearly separable but corrupted with adversarial label noise, despite the capacity to overfit.  To the best of our knowledge, this is the first work to show that overparameterized neural networks trained by SGD can generalize when the data is corrupted with adversarial label noise.  
\end{abstract}
\section{Introduction}


The remarkable ability of neural networks trained by stochastic gradient descent (SGD) to generalize, even when trained on data that has been substantially corrupted with random noise, seems at ends with much of contemporary statistical learning theory~\citep{zhang2017rethinkinggeneralization}.  How can a model class which is rich enough to fit randomly labeled data fail to overfit when a significant amount of random noise is introduced into the labels?  And how is it that a local optimization method like SGD is so successful at learning such model classes, even when the optimization problem is highly non-convex?

In this paper, we approach these questions by analyzing the performance of SGD-trained networks on distributions which can have substantial amounts of 
label noise.  For a distribution $\calD$ over features $(x,y)\in \R^d\times \{\pm 1\}$, let us define 
\begin{equation} \optlin := \min_{v\in \R^d,\ \norm{v}=1} \P_{(x,y)\sim \calD}\Big(y \neq \sgn\big(\sip{v}x\big)\Big) \label{def.optlin}
\end{equation}
as the optimal classification error achieved by a halfspace $\sip{v}\cdot$.   We prove that for a broad class of distributions, SGD-trained one-hidden-layer neural networks achieve classification error at most $\tilde O(\sqrt{\optlin})$ in polynomial time.  Equivalently, one-hidden-layer neural networks can learn halfspaces up to risk $\tilde O(\sqrt{\optlin})$ in the distribution-specific agnostic PAC learning setting.   Our result holds for neural networks with leaky-ReLU activations trained on the cross-entropy loss and, importantly, hold for any initialization, and for networks of arbitrary width. 

By comparing the generalization of the neural network with that of the \textit{best} linear classifier over the distribution, we can make two different but equally important claims about the training of overparameterized neural networks.  The first view is that SGD produces neural networks with classification error that is competitive with that of the best linear classifier over the distribution, and that this behavior can occur for neural networks of any width and any initialization.  In this view, our work provides theoretical support for the hypothesis put forward by~\citet{nakkiran2019sgd} that the performance of SGD-trained networks in the early epochs of training can be explained by that of a linear classifier.  

The second view is that of the problem of learning halfspaces in the presence of adversarial label noise.  (Note that adversarial \textit{label noise} is distinct from the notions of adversarial examples or adversarial training~\citep{goodfellow2014explaining,madry2018adversarial}
, where the features $x$ are perturbed rather than the labels $y$.) In this setting, one views the (clean) data as initially coming from a linearly separable distribution but for which each sample $(x,y)\sim \calD$ has its label flipped $y \mapsto -y$ with some sample-dependent probability $\eta(x)\in [0,1]$.     Then the best error achieved by a halfspace is $\E_{x\sim \calD_x}[\eta(x)]=\optlin$.  Viewed from this perspective, our result shows that despite the clear capacity of an overparameterized neural network to overfit to corrupted labels, when trained by SGD, such networks can still generalize (albeit achieving the suboptimal risk $\sqrt{\optlin}$).  We note that the optimization algorithm we consider is vanilla online SGD without any explicit regularization methods such as weight decay or dropout.  This suggests that the ability of neural networks to generalize in the presence of noise is not solely due to explicit regularization, but that some forms of \textit{implicit} regularization induced by gradient-based optimization play an important role.


\subsection{Related Work}

We discuss here a number of works related to the questions of optimization and generalization in deep learning.  An approach that has attracted significant attention recently is the neural tangent kernel (NTK) approximation~\citep{jacot2018ntk}.  This approximation relies upon the fact that for a specific initialization scheme, extremely wide neural networks are well-approximated by the behavior of the neural network at initialization, which in the infinite width limit produces a kernel (the NTK)~\citep{du2019-1layer,du2019deep,allenzhu2019convergence,zou2019gradient,cao2019generalization,arora2019exact,arora2019finegrained,cao2019generalizationsgd,frei2019resnet,zou2019improved,jitelgarsky20.polylog,chen.polylog}.  Using an assumption on separability of the training data, it is commonly shown that SGD-trained neural networks in the NTK regime can perfectly fit any training data.  Under certain conditions, one can also derive generalization bounds for the performance of SGD-trained networks for distributions that can be perfectly classified by functions related to the NTK.  

Although significant insights into the training dynamics of SGD-trained networks have come from this approach, it is known that neural networks deployed in practice can traverse far enough from their initialization such that the NTK approximation no longer holds~\citep{fort2020ntk}.  A line of work known as the mean field approximation allows for ultra-wide networks to be far from their initialization by connecting the trajectory of the weights of the neural network to the solution of an associated partial differential equation~\citep{mei2019mean,chizat2018note,chen2020generalized}.  A separate line of work has sought to demonstrate that the concept classes that can be learned by neural networks trained by gradient descent are a strict superset of those that can be learned by the NTK~\citep{allenzhu.kernel,wei2020regularization,li2020largelearningrate,woodworth2020kernelrich,li2020relubeyondntk}.  

More relevant to our work is understanding the generalization of neural network classifiers when the data distribution has some form of label noise.  Works that explicitly derive generalization bounds for SGD-trained neural networks in the presence of label noise are scarce. 
Even for the simple concept class of halfspaces $x\mapsto \sgn(\sip{v}x)$, there are often tremendous difficulties in determining whether or not \textit{any} algorithm can efficiently learn in the presence of noise.  For this reason let us take a small detour to detail some of the difficulties in learning halfspaces in the presence of noise, to emphasize the difficulty of learning more complicated function classes in the presence of noise.

The most general (and most difficult) noise class is that of adversarial label noise, which is equivalent to the agnostic PAC learning framework~\citep{kearns.agnostic}.  In this setting, one makes no assumption on the relationship between the features and the labels, and so continuing with the notation from \eqref{def.optlin}, the optimal risk $\optlin$ achieved by a halfspace is strictly positive in general.  It is known that learning up to classification error $O(\optlin)+\eps$ cannot be done in $\poly(d,\eps^{-1})$ time without assumptions on the marginal distribution of $\calD$~\citep{daniely2016complexity}.  For this reason it is common to assume some type of structure on the noise or the distribution to get tractable guarantees.  

One relaxation of the noise condition is known as the Massart noise~\citep{massart2006noise} where one assumes that each sample has its label flipped with some instance-dependent probability $\eta(x)\leq \eta < \nicefrac 12$.  Under this noise model, it was recently shown that there are efficient algorithms that can learn up to risk $\eta + \eps$~\citep{diakonikolas2019massart}.  A more simple noise setting is that of random classification noise (RCN)~\citep{angluin1988rcn}, where the labels of each sample are flipped with probability $\eta$.  Polynomial time algorithms for learning under this model were first shown by~\citet{blum1998rcn}.  Previous theoretical works on the ability of neural network classifiers to generalize in the presence of label noise were restricted to the RCN setting~\citep{hu2020simple} or Massart noise setting~\citep{li2019labelnoise}.  In this paper, we consider the most general setting of adversarial label noise.

In terms of distribution-specific learning guarantees in the presence of noise, polynomial time algorithms for learning halfspaces under Massart noise for the uniform distribution on the sphere were first shown by ~\citet{awasthi2015massart}, and for log-concave isotropic distributions by~\citet{awasthi20161bitcompressednoise}.~\citet{awasthi2017acm.localization} constructed a localization-based algorithm that efficiently learns halfspaces up to risk $O(\optlin)$ when the marginal is log-concave isotropic. 
For more background on learning halfspaces in the presence of noise, we refer the reader to~\citet{balcan2020noise}.  

Returning to the neural network literature, in light of the above it should not be surprising that computational tractability issues arise even for the case of neural networks consisting of a single neuron.~\citet{goel2019relugaussian} showed that learning a single ReLU neuron up to the best-possible risk $\opt_{\mathrm{ReLU}}$ (under the squared loss) is computationally intractable, even when the marginal is a standard Gaussian.  By contrast,~\citet{frei2020singleneuron} showed that gradient descent on the empirical risk can learn single ReLUs up to risk $O(\sqrt{\opt_{\mathrm{ReLU}}})$ efficiently for many distributions.   Two recent works have shown that even in the realizable setting---i.e., when the labels are generated by a neural network without noise---it is computationally hard 
to learn one-hidden-layer neural networks with (non-stochastic) gradient descent when the marginal distribution is Gaussian~\citep{goel2020superpolynomial,diakonikolas2020sqlowerbound}.

In terms of results that show neural networks can generalize in the presence of noise,~\citet{li2019labelnoise} considered clustered distributions with real-valued labels (using the squared loss) and analyzed the performance of GD-trained one-hidden-layer neural networks when a fraction of the labels are switched.  They derived guarantees for the empirical risk but did not derive a generalization bound for the resulting classifier.~\citet{hu2020simple} analyzed the performance of regularized neural networks in the NTK regime when trained on data with labels corrupted by RCN, and argued that regularization was helpful for generalization.  By contrast, our work shows that neural networks can generalize for linearly separable distributions corrupted by adversarial label noise without any explicit regularization, suggesting that certain forms of implicit regularization in the choice of the algorithm plays an important role.  We note that a number of researchers have sought to understand the implicit bias of gradient descent~\citep{soudry2018implicitbias,jitelgarsky2019implicit,lyuli2020implicitbias,ji2020directional,moroshko2020implicit,li2020implicitadversarial}.  Such works assume that the distribution is linearly separable by a large margin, and characterize the solutions found by gradient descent (or gradient flow) in terms of the maximum margin solution.

Finally, we note some recent works that connected the training dynamics of SGD-trained neural networks with linear models.~\citet{brutzkus2018sgd} showed that SGD-trained one-hidden-layer leaky ReLU networks can generalize on linearly separable data.~\citet{shamir2018resnetslinear} compared the performance of residual networks with those of linear predictors in the regression setting.  They showed that there exist weights for residual networks with generalization performance competitive with linear predictors, and they proved that SGD is able to find those weights when there is a residual connection from the input layer to the output layer.  ~\citet{nakkiran2019sgd} provided experimental evidence for the hypothesis that much of the performance of SGD-trained neural networks in the early epochs of training can be explained by linear classifiers.~\citet{hu2020surprising} provided theoretical evidence for this hypothesis by showing that overparameterized neural networks with the NTK initialization and scaling have similar dynamics to a linear predictor defined in terms of the network's NTK.~\citet{shah2020pitfalls} showed that neural networks are biased towards simple classifiers even when more complex classifiers are capable of improving generalization.

\section{Problem Description and Results}
In this section we study the problem we consider and our main results.

\subsection{Notation}
For a vector $v$, we denote $\norm{v}$ as its Euclidean norm.  For a matrix $W$, we use $\pnorm WF$ to denote its Frobenius norm.  We use the standard $O(\cdot)$ and $\Omega(\cdot)$ notations to ignore universal constants when describing growth rates of functions.  The notation $\tilde O(\cdot)$ and $\tilde \Omega(\cdot)$ further ignores logarithmic factors.  We use $a\vee b$ to denote the maximum of $a,b\in \R$, and $a \wedge b$ their minimum.  The notation $\ind(E)$ denotes the indicator function of the set $E$, which is one on the set and zero outside of it.  

\subsection{Problem Setup}
Consider a distribution $\calD$ over $(x,y)\in \R^d\times \{\pm 1\}$ with marginal distribution $\calD_x$ over $x$.  Let $m\in \N$, and consider a one-hidden-layer leaky ReLU network with $m$ neurons,
\begin{equation}
    f_x(W) := \summ j {m}a_j \sigma(\sip{w_j}x),\label{eq:leaky.relu.def}
\end{equation} 
where $\sigma(z) = \max(\alpha z, z)$ is the leaky-ReLU activation with $\alpha\in (0,1]$.  Assume that $a_j \iid \mathrm{Unif}(\pm a)$ for some $a>0$ and that the $\{a_j\}$ are randomly initialized and not updated throughout training, as is commonly assumed in theoretical analyses of SGD-trained neural networks~\citep{du2019-1layer,arora2019finegrained,jitelgarsky20.polylog}.
\footnote{The specific choice of the initialization of the second layer is immaterial; our analysis holds for any second-layer weights that are fixed at a random initialization.  The only difference that may arise is in the sample complexity: if with high probability $\norm{a}= \Theta (1)$ then the sample complexity requirement will be the same within constant factors, while for initializations satisfying $\norm{a}= \omega(1)$ or $\norm{a}=o(1)$ our upper bound for the sample complexity will become worse as the network becomes larger.}  
We are interested in the classification error for the neural network,
\[ \err(W) := \P_{(x,y)\sim \calD} \Big(y \neq \sgn\big(f_x(W)\big)\Big),\]
where $\sgn(z)=1$ if $z >0$, $\sgn(0)=0$, and $\sgn(z)=-1$ otherwise.  
We will seek to minimize $\err(W)$ by minimizing,
\[ L(W) := \E_{(x,y)\sim \calD} \ell( y f_x(W)),\]
where $\ell$ is a convex loss function.  
We will use the fact that for any convex, twice differentiable and decreasing function $\ell$, the function $-\ell'$ is non-negative and decreasing, and thus $-\ell'$ can also serve as a loss function.  In particular, by Markov's inequality, these properties allow us to bound the classification error by the population risk under $-\ell'$:
\begin{align} \nonumber\P_{(x,y)\sim \calD} \Big(y \neq \sgn\big(f_x(W)\big)\Big) &= \P\Big(y\cdot  f_x(W) \leq 0 \Big) \\ \nonumber
&=\P\Big( -\ell'\big(y f_x(W) \big)  \geq 0 \Big) \\
&\leq \frac{\E_{(x,y)\sim \calD}-\ell'\big( y f_x(W) \big)}{-\ell'(0)} \label{eq:markov.inequality.calc}
\end{align}
Thus, provided $-\ell'(0)>0$, upper bounds for the population risk under $-\ell'$ yield guarantees for the classification error.  This property has previously been used to derive generalization bounds for deep neural networks trained by gradient descent~\citep{cao2019generalization,frei2019resnet,jitelgarsky20.polylog,chen.polylog}. To this end, we make the following assumptions on the loss throughout this paper.
\begin{assumption}\label{assumption:loss}
The loss $\ell(\cdot):\R\to \R$ is convex, twice differentiable, decreasing, 1-Lipschitz, and satisfies $-\ell'(0) > 0$.  Moreover, for $z\geq 1$, $\ell$ satisfies $-\ell'(z) \leq 1/z$.  
\end{assumption}
The assumption that $-\ell'(z)\leq 1/z$ for $z\geq 1$ is to ensure that the surrogate loss $-\ell'$ is not too large on samples that are classified correctly.  Note that the standard loss used for training neural networks in binary classification tasks---the binary cross-entropy loss $\ell(z) = \log(1+\exp(-z))$---satisfies all of the conditions in Assumption \ref{assumption:loss}.   We denote the population risk under the surrogate loss $-\ell'$ as follows,
\[ \calE(W) := \E_{(x,y)\sim \calD} -\ell'(y f_x(W)).\]
We seek to minimize the population risk by minimizing the empirical risk induced by a set of i.i.d. examples $\{(x_t, y_t) \}_{t\geq 1}$ using the online stochastic gradient descent algorithm.  Denote $f_{t}(W) = f_{x_t}(W)$ as the neural network output for sample $x_t$, and denote the loss under $\ell$ and $-\ell'$ for sample $x_t$ by
\begin{align}
    \hat L_t(W) := \ell(y_t f_t(W)),\quad \hat \calE_t(W) := -\ell'(y_t f_t(W)).
\end{align}
The updates of online stochastic gradient descent are given by
\begin{equation}
    \Wt {t+1} := \Wt t - \eta \nabla \hat L_t(\Wt t) = \Wt t - \eta \ell'(y_t f_t(\Wt t)) y_t \nabla f_t(\Wt t) = \Wt t + \eta \hat \calE_t(\Wt t) y_t \nabla f_t(\Wt t).
\end{equation}

Before proceeding with our main theorem we will introduce some of the definitions and assumptions which will be used in our analysis.  The first is that of sub-exponential distributions. 
\begin{definition}[Sub-exponential distributions]\label{def:concentration}
We say $\calD_x$ is $C_m$\emph{-sub-exponential} if every $x\sim \calD_x$ is a sub-exponential random vector with sub-exponential norm at most $C_m$.  In particular, for any $\bar v \in \R^d$ with $\norm {\bar v}=1$, $\P_{\calD_x}(|\bar v^\top x| \geq t) \leq \exp(- t/C_m)$.  
\end{definition}
We note that every sub-Gaussian distribution is sub-exponential.  The next property we introduce is that of a \textit{soft margin}.  This condition was recently utilized by~\citet{frei2020halfspace} for the agnostic learning of halfspaces using convex surrogates for the zero-one loss.
\begin{definition}\label{def:softmargin}
Let $\bar v\in \R^d$ satisfy $\norm{\bar v} =1$.  We say $\bar v$ satisfies the \emph{soft margin condition with respect to a function $\phi_{\bar v}:\R \to \R$} if for all $\gamma \in [0,1]$, it holds that
\[ \E_{x\sim \calD_x}\l[\ind\l(   |\bar v^\top x| \leq \gamma \r)\r]\leq \phi_{\bar v}(\gamma).\]
\end{definition}
The soft margin can be seen as a probabilistic analogue of the standard hard margin, where we relax the typical requirement for a margin-based condition from holding almost surely to holding with some controlled probability.  
As written above, the soft margin condition can hold for a specific vector $\bar v\in \R^d$, and our final generalization bound below will only care about the soft margin function for a halfspace $\bar v$ that achieves population risk $\optlin$.  However, for many distributions, one can show that \textit{all} unit norm vectors $\bar v$ satisfy a soft margin of the form $\phi_{\bar v}(\gamma) = O(\gamma)$.  One important class of such distributions are those satisfying a type of anti-concentration property.
\begin{definition}[Anti-concentration]\label{assumption:anti.concentration}
For $\bar v\in \R^d$, denote by $p_{\bar v}(\cdot)$ the marginal distribution of $x\sim \calD_x$ on the subspace spanned by $\bar v$.  We say $\calD_x$ satisfies \emph{$U$-anti-concentration} if there is some $U>0$ such that for all unit norm $\bar v$, $p_{\bar v}(z)\leq U$ for all $z\in \R$.
\end{definition}
Anti-concentration is a typical assumption used for deriving distribution-specific agnostic PAC learning guarantees~\citep{klivans2009maliciousnoise,diakonikolas2020massartstructured,diakonikolas2020nonconvex,frei2020halfspace} as it allows for one to ignore pathological distributions where arbitrarily large probability mass can be concentrated in tiny regions of the domain.   
Below, we collect some examples of soft margin function behavior for different distributions, including those satisfying the above anti-concentration property.   We shall see in Theorem \ref{thm:online.sgd.leaky.relu} that the behavior of $\phi(\gamma)$ for $\gamma\ll 1$ will be the determining factor in our generalization bound, and thus in the below examples one only needs  to pay attention to the behavior of $\phi(\gamma)$ for $\gamma$ sufficiently small.
\begin{example}\label{example:soft.margin}
\begin{enumerate}
    \item If $|\bar v^\top x| > \gamma^*$ a.s., then $\phi_{\bar v}(\gamma) = 0$ for $\gamma < \gamma^*$.
    \item If $\calD_x$ satisfies $U$-anti-concentration, then for any $\bar v$ with $\norm{\bar v}=1$, $\phi_{\bar v}(\gamma) \leq 2 U \gamma$ holds.
    \item If $\calD_x$ is isotropic and log-concave (i.e. its probability density function is log-concave), then $\calD_x$ satisfies $1$-anti-concentration and hence $\phi_{\bar v}(\gamma) \leq 2 \gamma$ for all $\bar v$. 
\end{enumerate}
\end{example}
The proofs for the properties described in Example \ref{example:soft.margin} can be found in~\citet[Section 3]{frei2020halfspace}.  

\subsection{Main Results}
With the above in place, we can provide our main result.  

\begin{theorem}\label{thm:online.sgd.leaky.relu}
Assume $\calD_x$ is $C_m$-subexponential and there exists $B_X>0$ such that $\E[\norm{x}^2]\leq B_X^2<\infty$.  Denote $\optlin := \min_{\norm w =1} \P_{(x,y)\sim \calD}(y \sip{w}x < 0)$ as the best classification error achieved by a unit norm halfspace $v^*$.  Let $m\in \N$ be arbitrary, and consider a leaky-ReLU network of the form \eqref{eq:leaky.relu.def} where $a = 1/\sqrt{m}$.  Let $\Wt 0$ be an arbitrary initialization and denote $G_0 := \|\Wt 0\|_F$.     Let the step size satisfy $\eta \leq B_X^{-2}$.  Then for any $\gamma>0$, by running online SGD for $T = O(\eta^{-1} \gamma^{-2} [\phi_{v^*}(\gamma) + \optlin]^{-2} [1\vee G_0])$ iterations, there exists a point $t^*<T$ such that in expectation over $(x_1, \dots, x_T)\sim \calD^T$,
\[ \P_{(x,y)\sim \calD} \Big(y \neq \sgn\big( f_x(\Wt {t^*})\big) \Big)  \leq 2|\ell'(0)|^{-1} \alpha^{-1} \bigg[ \Big(1 + \gamma^{-1} C_m + \gamma^{-1} C_m \log(\nicefrac 1{\optlin}) \Big)\optlin + \phi_{v^*}(\gamma) \bigg].\]
\end{theorem}
To concretize the generalization bound in Theorem \ref{thm:online.sgd.leaky.relu} we need to analyze the properties of the soft margin function $\phi_{v^*}$ at the best halfspace and then optimize over the choice of $\gamma$.  But before doing so, let us make a few remarks on Theorem \ref{thm:online.sgd.leaky.relu} that hold in general.  The sample complexity (number of SGD iterations) $T$, and the resulting generalization bound, are independent of the number of neurons $m$, showing that the neural network can generalize despite the capacity to overfit.
\footnote{\citet[Theorem 7]{brutzkus2018sgd} showed that if there are $T$ samples and $m = \Omega(T/d)$, then for any set of labels $(y_1, \dots, y_T)\in \{\pm 1\}^T$ and for almost every $(x_1, \dots, x_T)\sim \calD_x^T$, there exist hidden layer weights $W^*$ and outer layer weights $\vec a\in \R^m$ such that $f_{t}(W^*) = y_t$ for all $t\in [T]$.  In contrast, Theorem \ref{thm:online.sgd.leaky.relu} shows that when $m$ is sufficiently large there exist neural networks that can fit random labels of the data but SGD training avoids these networks.}  If $\norm{x}\leq B_X$ a.s. for some absolute constant $B_X$, then the sample complexity is dimension-independent, while if $\calD_x$ is isotropic, $\E[\norm{x}^2]=d$ and so the sample complexity is linear in $d$.  Finally, we note that large learning rates and arbitrary initializations are allowed.
  
In the remainder of the section, we will discuss the implications of Theorem \ref{thm:online.sgd.leaky.relu} for common distributions. The first distribution we consider is a hard margin distribution.  

\begin{corollary}[Hard margin distributions]\label{corollary:hard.margin}
Suppose there exists some $v^* \in \R^d$, $\norm{v^*}=1$, and $\gamma_0>0$ such that $\P\big(y\neq \sgn(\sip{v^*}x)\big)= \optlin$ and $|\sip{v^*}x|\geq \gamma_0>0$ almost surely over $\calD_x$.  Assume for simplicity that $\ell$ is the binary cross-entropy loss, $\ell(z)=\log(1+\exp(-z))$.  Then under the settings of Theorem \ref{thm:online.sgd.leaky.relu}, there exists some $t^*<T = O(\eta^{-1} \gamma_0^{-2} \optlin^{-2} [1\vee G_0])$ such that in expectation over $(x_1, \dots, x_T)\sim \calD^T$,
\[ \P_{(x,y)\sim \calD} \Big(y \neq \sgn\big( f_x(\Wt {t^*})\big) \Big) \leq  \tilde O(\gamma_0^{-1} \optlin).\]
\end{corollary}
\begin{proof}
Since $|\sip{v^*}x|\geq \gamma_0>0$, the soft margin at $v^*$ satisfies $\phi_{v^*}(\gamma_0)=0$.  Since $-\ell'(0)=\nicefrac 12$, by Theorem \ref{thm:online.sgd.leaky.relu},
\[  \P_{(x,y)\sim \calD} \Big(y \neq \sgn\big( f_x(\Wt {t^*})\big) \Big) \leq 4 \alpha^{-1}\big( 1 + \gamma_0^{-1} C_m + \gamma_0^{-1} C_m \log(\nicefrac 1 \optlin) \big) \optlin. \]
\end{proof}

The above result shows that if the data comes from a linearly separable data distribution with margin $\gamma_0$ but is then corrupted by adversarial label noise, then SGD-trained networks will still find weights that can generalize with classification error at most $\tilde O(\gamma_0^{-1}\optlin)$.
In the next corollary we show that for distributions satisfying $U$-anti-concentration we get a generalization bound of the form $\tilde O(\sqrt{\optlin})$.

\begin{corollary}[Distributions satisfying anti-concentration]\label{corollary:anti.concentration}
Assume $\calD_x$ satisfies $U$-anti-concentration.   Assume for simplicity that $\ell$ is the binary cross-entropy loss, $\ell(z)=\log(1+\exp(-z))$.  Then under the settings of Theorem \ref{thm:online.sgd.leaky.relu}, there exists some $t^*<T = O(\eta^{-1} \optlin^{-3} [1\vee G_0])$ such that in expectation over $(x_1, \dots, x_T)\sim \calD^T$,
\[ \P_{(x,y)\sim \calD} \Big(y \neq \sgn\big( f_x(\Wt {t^*})\big) \Big) \leq \tilde O(\sqrt{\optlin}).\]
\end{corollary}
\begin{proof}
By Example \ref{example:soft.margin}, $\phi_{v^*}(\gamma) \leq 2 U \gamma$.  Substituting this into Theorem \ref{thm:online.sgd.leaky.relu} and using that $-\ell'(0)=\nicefrac 12$, we get
\[ \P_{(x,y)\sim \calD} \Big(y \neq \sgn\big( f_x(\Wt {t^*})\big) \Big) \leq 4 \alpha^{-1}\big [ 2 U \gamma + 3 \gamma^{-1} C_m \opt \log(\nicefrac 1 {\opt})\big].\]
This bound is optimized when $\gamma = \opt^{1/2}$, and results in a bound of the form $O(\opt^{1/2} \log(\nicefrac 1 {\opt}))$.
\end{proof}
The above corollary covers, for instance, log-concave isotropic distributions like the Gaussian or the uniform distribution over a convex set by Example \ref{example:soft.margin}. 

Taken together, Corollaries \ref{corollary:hard.margin} and \ref{corollary:anti.concentration} demonstrate that despite the capacity for overparameterized neural networks to overfit to the data, SGD-trained neural networks are fairly robust to adversarial label noise.  We emphasize that our results hold for SGD-trained neural networks of arbitrary width and following an arbitrary initialization, and that the resulting generalization and sample complexity do not depend on the number of neurons $m$.  In particular, the above phenomenon cannot be explained by the neural tangent kernel approximation, which is highly dependent on assumptions about the initialization, learning rate, and number of neurons.

\subsection{Comparisons with Related Work}
We now discuss how our result relates to others appearing in the literature.  First,
\citet{brutzkus2018sgd} showed that by running multiple-pass SGD on the hinge loss one can learn linearly separable data.  They assume a noiseless $(\optlin=0)$ model over a norm-bounded domain and assume a hard margin distribution, so that $y\sip{v^*}x > \gamma_0$ for some $\gamma_0>0$.  In the noiseless setting, Corollaries~\ref{corollary:hard.margin} and \ref{corollary:anti.concentration} generalize their result to include unbounded, linearly separable (marginal) distributions without a hard margin like log-concave isotropic distributions.  More significantly, our results hold in the adversarial label noise setting (a.k.a., agnostic PAC learning).  This allows for us to compare the generalization of an SGD-trained neural network with that of the \textit{best} linear classifier over the distribution, and make a much more general claim about the dynamics of SGD-trained neural networks.

\citet{hu2020surprising} showed that for sufficiently wide neural networks with the NTK initialization scheme, and under the assumption that the components of the input distribution are independent, the dynamics in the early stages of SGD-training are closely related to that of a linear predictor defined in terms of the NTK of the neural network.  By contrast, our result holds for any initialization and neural networks of any width and covers a larger class of distributions.  Their result was for the squared loss, while ours holds for the standard losses used for classification problems.  Our results can be understood as a claim about the `early training dynamics' of SGD, since we show that there exists \textit{some} iterate of SGD that performs almost as well as the best linear classifier over the distribution, and we provide an upper bound on the number of iterations required to reach this point.  One might expect that under more stringent assumptions (on, say, the initialization, learning rate schedule, and/or network architecture), stronger guarantees for the classification error could hold in the later stages of training; we will revisit this question with experimental results in Section \ref{sec:experiment}.

\citet{li2019labelnoise} considered a handcrafted distribution consisting of noisy clusters and showed that sufficiently wide one-hidden-layer neural networks trained by GD on the squared loss with the NTK initialization have favorable properties in the early training dynamics.   A direct comparison of our results is difficult as they do not provide a guarantee for the generalization error of the resulting neural network.  But at a high level, their analysis focused on a noise model akin to Massart noise (a more restrictive setting than the agnostic noise considered in this paper), and they made a number of assumptions---a particular (large) initialization, sufficiently wide network, and the use of the squared loss for classification---that were not used in this work.  The results of~\citet{li2019labelnoise} covered general, smooth activation functions (but not leaky-ReLU). 

\citet{hu2020simple} showed that ultra-wide networks with NTK scaling and initialization trained by SGD with various forms of regularization can generalize when the labels are corrupted with random classification noise.  Their generalization bound was given in terms of the classification error on the `clean' data distribution (without any noise) and allowed for general activation functions (including leaky-ReLU).    In comparison, we assume that the training data and the test data come from the same distribution, and our generalization bound is given in terms of the performance of the best linear classifier over the distribution.  Our generalization guarantee holds without any explicit forms of regularization, suggesting that the mechanism responsible for the lack of overfitting is not explicit regularization, but forms of regularization that are \textit{implicit} to the SGD algorithm.  

\section{Proof of the Main Results}
We will show that stochastic gradient descent achieves small classification error by using a proof technique similar to that of~\citet{brutzkus2018sgd}, who showed the convergence and generalization of gradient descent on the hinge loss for one-hidden-layer leaky ReLU networks on linearly separable data.\footnote{This proof technique can be viewed as an extension of the Perceptron proof presented in~\citet[Theorem 9.1]{shalevschwartz}.}   Their proof relies upon the fact that both the classification error and the hinge loss for the best halfspace are zero.  In our setting---without the assumption of linear separability, and with more general loss functions---their strategy for showing that the empirical risk can be driven to zero will not work.  (We remind the reader that our goal is to show that the neural network will generalize when it is of \textit{arbitrary} width, and when significant noise is present, and thus we cannot guarantee the smallest empirical or population loss is arbitrarily close to zero.)  Instead, we need to compare the performance of the neural network with that of the best linear classifier over the data, which will in general have error (both classification and loss value) bounded away from zero.  To do so, we use some of the ideas used in~\citet{frei2020halfspace} to derive generalization bounds for the classification error when the surrogate loss is bounded away from zero. 

To begin, let us introduce some notation.  Let $v^*\in \R^d$ be a unit norm halfspace that minimizes the halfspace error, so that
\[ \P_{(x,y)\sim \calD} \Big( y \neq \sgn\big(\sip{v^*}x\big)\Big) =  \optlin.\]
Denote the matrix $V\in \R^{m\times d}$ as having rows $v_j^\top \in \R^d$ defined by
\begin{equation} \label{eq:vdef}
v_j = \f 1 {\sqrt {m}} \sgn(a_j) v^*.
\end{equation}
The scaling of each row of the matrix $V$ ensures that $\pnorm{V}{F}=1$.  For $\gamma>0$, denote
\begin{equation*}
    \hat \xi_t(\gamma) := \ind(y_t \sip{v^*}{x_t} \in [0,\gamma)) + (1 +\gamma^{-1} |\sip{v^*}{x_t}|) \ind(y_t \sip{v^*}{x_t} < 0).
\end{equation*}
The expected value of the above quantity will be an important quantity in our proof.  To give some idea of how this quantity will fit in to our analysis, assume for the moment that $\norm{x}\leq 1$ a.s. Then taking expectations of the above and using Cauchy--Schwarz, we get
\begin{equation}\E \hat \xi_t(\gamma) \leq \phi_{v^*}(\gamma) + (1 + \gamma^{-1}) \E[|\sip{v^*}{x_t}| \ind(y_t \sip{v^*}{x_t} < 0)] \leq \phi_{v^*}(\gamma) + (1+ \gamma^{-1}) \optlin.\label{eq:xi.def}
\end{equation}
The above appears (in a more general form) in the bound for the classification error presented in Theorem~\ref{thm:online.sgd.leaky.relu}.  In particular, the goal below will be to show that the classification error can be bounded by a constant multiple of $\E[\hat \xi_t(\gamma)]$.

Continuing, let us denote
\begin{equation}\label{eq:def.hath}
    \hat H_t := \sip{\Wt t}V,\quad \hat G_t^2 = \pnorm{\Wt t}F^2.
\end{equation}
The quantity $\hat H_t$ measures the correlation between the weights found by SGD and those of the best linear classifier over the distribution.   We define the population-level versions of each of the random variables above by replacing the $\hat \cdot$ with their expectation $\esgd(\cdot)$ over the randomness of the draws $(x_1, \dots, x_t)$ of the distribution used for SGD.  That is,
\begin{align}\nonumber
    L_t &:= \esgd \hat L_t (\Wt t),\\\nonumber
    \calE_t &:= \esgd \hat \calE_t(\Wt t),\\\nonumber
    H_t &:= \esgd \hat H_t,\\\nonumber
    G_t^2 &:= \esgd [\hat G_t^2],\\
    \xi(\gamma) &:= \E_{(x_t,y_t)\sim \calD} \hat \xi_t(\gamma).\label{eq:H.G.def}
\end{align}

Our proof strategy will be to show that until gradient descent finds weights with small risk, the correlation $H_T$ between the weights found by SGD and those of the best linear predictor will grow at least as fast as $\Omega(T)$, while $G_T$ always grows at a rate of at most $O(\sqrt{T})$. Since $\pnorm VF=1$, by Cauchy--Schwarz we have the bound $H_T \leq G_T$, and so the growth rates $H_T = \Omega(T)$ and $G_T = O(\sqrt{T})$ can only be satisfied for a small number of iterations.  In particular, there can only be a small number of iterations until SGD finds weights with small risk.

To see how we might be able to show that the correlation $H_T$ is increasing, note that we have the identity
\begin{align*}\hat H_{t+1} -\hat H_t&= - \eta \sip{\nabla \hat L_t(\Wt t)}V = - \eta \ell'(y_t f_t(\Wt t)) y_t \sip{\nabla f_t(\Wt t)}V.\end{align*}
Since $-\ell'\geq 0$, the inequality $\hat H_{t+1}>\hat H_t$ holds if we can show $y_t \sip{\nabla f_t(\Wt t)}{V}>0$, i.e. if we can show that the gradient of the neural network is correlated with the weights of the best linear predictor.   For this reason, the following technical lemma is a key ingredient in our proof.
\begin{lemma}\label{lemma:key.identity}
For $V$ defined in \eqref{eq:vdef}, for any $(x_t,y_t)\in \R^d \times \{ \pm 1\}$, for any $W\in \R^{m\times d}$, and any $\gamma\in (0,1)$,
\begin{equation}\label{eq:key.identity.strictly.increasing.lipschitz}
     y_t\sip{\nabla f_t(W)}{V} \geq a  \gamma \sqrt{m} \Big[ \alpha - \hat \xi_t(\gamma)\Big].
\end{equation}
\end{lemma}
The proof of the above lemma is in Appendix \ref{appendix:key.identity.proof}.  As alluded to above, with this technical lemma we can show that until the surrogate risk is as small as a constant factor of $\xi(\gamma)$, the correlation of the weights found by SGD and those of the best linear predictor is increasing. 

\begin{lemma}\label{lemma:H.lowerbound.sgd}
For any $t\in \N\cup \{0\}$, for any $\gamma>0$, it holds that 
\[ H_{t+1} \geq  H_t + \eta a \gamma \sqrt{m} \big[\alpha  \calE_t -\xi(\gamma)\big].\]
\end{lemma}
\begin{proof}
Since $\hat H_{t+1} = \sip{\Wt {t+1}}V = \sip{\Wt t}V - \eta \sip{\hat L_t(\Wt t)}V$, we can write
\begin{align*}
    \hat H_{t+1} &= \hat H_t - \eta \sip{\nabla \hat L_t(\Wt t)}V \\
    &= \hat H_t - \eta \ell'(y_t f_t(\Wt t)) y_t \sip{\nabla f_t(\Wt t)}V \\
    &\geq \hat H_t - \eta \ell'(y_t f_t(\Wt t)) a \gamma \sqrt m [\alpha - \hat \xi_t(\gamma)]\\
    &\geq \hat H_t + \eta a \gamma \sqrt m \l[ \alpha \hat \calE_t(\Wt t) - \hat \xi_t(\gamma)\r].
\end{align*}
 In the first inequality we have used Lemma \ref{lemma:key.identity} and that $-\ell'\geq 0$, and in the second inequality we have used that $-\ell' \leq 1$. 
Taking expectations over the draws of the distribution on both sides completes the proof. 
\end{proof}

Notice that if $\alpha \calE_t > \xi(\gamma)$, Lemma \ref{lemma:H.lowerbound.sgd} shows that $H_{t+1} - H_t >0$.  We will later repeat this argument for $T$ iterations to show that until we find a point with $\alpha \calE_t \leq 2 \xi(\gamma)$, $H_T$ will grow at least as fast as $\Omega(T)$.  

All that remains is to show that $G_T = O(\sqrt{T})$.  We will accomplish this by first demonstrating a bound on $G_{t+1}^2 - G_{t}^2$.  
\begin{lemma}\label{lemma:grad.ub.sgd}
For any $t\in \N \cup \{0\}$, $\eta >0$, and if $\E[\norm{x}^2]\leq B_X^2$, 
\[G_{t+1}^2 \leq G_t^2 + 2 \eta +  \eta^2 m a^2 B_X^2.\]
\end{lemma}
\begin{proof}
We begin with the identity
\begin{align}
    \hat G_{t+1}^2 &= \pnorm{\Wt t- \eta \nabla \hat L_t(\Wt t)}F^2 = \pnorm{\Wt t}F^2 - 2 \eta \ip{\Wt t}{\nabla \hat L_t(\Wt t)} - \eta^2 \pnorm{\nabla \hat L_t(\Wt t)}F^2.\label{eq:gradient.ub.init.sgd}
\end{align}
We proceed by analyzing the last two terms. We have
\begin{align*}
    \sip{\Wt t}{\nabla \hat L_t(\Wt t)} &= \ell'(y_t f_t(\Wt t)) y_t \sip{\Wt t}{\nabla f_{x_t}(\Wt t)} \\
    &= \ell'(y_t f_t(\Wt t))  y_t \summ j {m} a_j \sigma'(\sip{\wt t_j}{x_t}) \sip{\wt t_j}{x_t} \\
    &= \ell'(y_t f_t(\Wt t)) y_t \summ j {m}  a_j \sigma(\sip{\wt t_j}{x_t}) \\
    &=  \ell'(y_t f_t(\Wt t)) y_t f_t(\Wt t)).
\end{align*}
The third equality uses that $\sigma$ is homogeneous, so $\sigma'(z) z = \sigma(z)$.  We can therefore bound
\begin{equation}\label{eq:grad.ub.second.term.sgd}
- 2 \eta \ip{\Wt t}{\nabla \hat L_t(\Wt t)} =  -2\eta \ell'(y_t f_t(\Wt t)) y_t f_t(\Wt t) \leq 2\eta.
\end{equation}
To see that the inequality holds, note that $-\ell'(z)\cdot z\leq 1$ if $z\leq 1$ since $-\ell'(z)\in [0,1]$, and if $z\geq 1$ then $-\ell'(z)\leq 1/z$ by Assumption \ref{assumption:loss}.  
For the gradient norm term, if we denote $\vec a\in \R^m$ as the vector with $j$-th entry $a_j$ and $\Sigma_t^W\in \R^{m\times m}$ as the diagonal matrix with $j$-th diagonal entry $\sigma'(\sip{w_j}{x_t})$, then
\begin{align}\nonumber
    \pnorm{\nabla \hat L_t(W)}F^2 &= \pnorm{ \ell'(y_t f_t(W)) \Sigma_t^W \vec a x_t^\top}F^2 \\ \nonumber
    &= \ell'(y_t f_t(W))^2 \pnorm{\Sigma_t^W \vec a}2^2 \norm{x_t}^2\\
    &\leq m a^2 \norm{x_t}^2.\label{eq:grad.ub.third.term.sgd}
\end{align}
The second equation
uses that $\pnorm{bd^\top}F = \pnorm{b}2\pnorm{d}2$ for vectors $b,d$.  The inequality uses that $|\ell'|\in [0,1]$. 

Substituting \eqref{eq:grad.ub.second.term.sgd} and \eqref{eq:grad.ub.third.term.sgd} into \eqref{eq:gradient.ub.init.sgd}, we get
\begin{equation*}
    \hat G_{t+1}^2 \leq \hat G_t^2 + 2 \eta + m a^2 \eta^2 \norm{x_t}^2.
\end{equation*}
Taking expectations of both sides over the draws of the distribution we get
\begin{equation*}
    G_{t+1}^2 \leq G_t^2 + 2 \eta + m a^2 \eta^2 B_X^2,
\end{equation*}
where we have used that $\E[\norm{x}^2]\leq B_X^2$. 
\end{proof}

We now have all of the ingredients needed to prove Theorem \ref{thm:online.sgd.leaky.relu}.  
\begin{proof}[Proof of Theorem \ref{thm:online.sgd.leaky.relu}]
First, let us note that for $V$ defined as \eqref{eq:vdef} (satisfying $\pnorm VF=1$), we have by Cauchy--Schwarz,
\begin{equation}\label{eq:h.vs.g.inequality.cs.sgd}
H_t^2 = (\E[\sip{\Wt t}V])^2 \leq \E\snorm{\Wt t}_F^2 \E\pnorm{V}F^2 = G_t^2 \iff |H_t| \leq G_t.
\end{equation}
For $a = 1/\sqrt m$, and for $\eta \leq (m a^2 B_X^2)^{-1} = B_X^{-2}$, Lemma \ref{lemma:grad.ub.sgd} becomes
\begin{equation*}
    G_{t+1}^2 \leq G_t^2 + 2 \eta + \eta^2 m a^2 B_X^2 \leq G_t^2 + 3 \eta.
\end{equation*}
Summing the above from $t=0, \dots, T-1$, we get
\begin{equation}\label{eq:grad.ub.1/sqrtk}
    G_T^2 \leq G_0^2 + 3 \eta T.
\end{equation}
Similarly, Lemma \ref{lemma:H.lowerbound.sgd} becomes
\begin{equation}\nonumber
    H_{t+1} \geq H_t + \eta \gamma[\alpha \calE_t - \xi].
\end{equation}
(Note that $\xi = \xi(\gamma)$ depends on $\gamma$, but we have dropped the notation for simplicity.)  Summing the above, we get
\begin{equation}\label{eq:H.lowerbound.1/sqrtk}
    H_T \geq H_0 + \eta \gamma \summm t 0 {T-1}[\alpha \calE_t - \xi].
\end{equation}
We can therefore bound
\begin{align}\nonumber
    -G_0 + \eta \gamma \summm t 0 {T-1} [\alpha \calE_t - \xi] &\leq H_0 + \eta   \gamma  \summm t 0 {T-1} [\alpha \calE_t - \xi]\\ \nonumber
    &\leq H_T \\ \nonumber
    &\leq G_T \\
    &\leq G_0 + \sqrt{T} \cdot 2 \sqrt \eta. \label{eq:sgd.final.ineq}
\end{align}
The first inequality uses~\eqref{eq:h.vs.g.inequality.cs.sgd}.  The second inequality uses \eqref{eq:H.lowerbound.1/sqrtk}.  The third inequality again uses~\eqref{eq:h.vs.g.inequality.cs.sgd}.  The final inequality uses \eqref{eq:grad.ub.1/sqrtk} together with $\sqrt{a+b}\leq \sqrt a + \sqrt b$.  

We claim now that this implies that within a polynomial number of samples, SGD finds weights satisfying $\calE_t \leq 2 \alpha^{-1} \xi$.  Suppose that for every iteration $t=1, \dots, T$, we have $\calE_t > 2\alpha^{-1}  \xi$.  Then~\eqref{eq:sgd.final.ineq} gives
\begin{equation*}
    \eta \alpha \gamma \xi T \leq 2 G_0 + 2 \sqrt \eta \cdot \sqrt T \iff \eta \alpha \gamma \xi \cdot T - 2 \sqrt \eta \cdot \sqrt T - 2 G_0 \leq 0.
\end{equation*}
This is an equation of the form $\beta_2 x^2 -\beta_1 x -\beta_0 \leq 0$, and thus using the quadratic formula, this implies $\sqrt T \leq (2 \beta_2)^{-1}(-\beta_1 + \sqrt{\beta_1^2 - 4 \beta_0 \beta_2})$.  
Squaring both sides and using a bit of algebra, this implies
\[ T \leq  \beta_2^{-2} \beta_1^2 + \beta_2^{-3/2} \beta_1 \beta_0^{1/2} + \beta_2^{-1}\beta_0.\]
In particular, we have
\begin{align*}
    T &\leq \eta^{-2} \alpha^{-2} \gamma^{-2} \xi^{-2}\cdot 4 \eta  +  \eta^{-3/2} \alpha^{-3/2} \gamma^{-3/2} \xi^{-3/2} \cdot 2 \eta^{1/2} \cdot G_0^{1/2} + \eta^{-1} \alpha^{-1} \xi^{-1} \cdot 2 G_0 \\
    &\leq 4 \eta^{-1} \alpha^{-2} \gamma^{-2} \xi^{-2} (G_0\vee 1).
\end{align*}
That is, within $T = O(\eta^{-1} \gamma^{-2} \xi^{-2} [G_0\vee 1])$ iterations, gradient descent finds a point satisfying
\begin{equation} \label{eq:sgd.guarantee}
\calE_t = \esgd \big[ -\ell'\big(y f_x(\Wt t)\big) \big] \leq 2\alpha^{-1} \xi.
\end{equation}
By Markov's inequality (see \eqref{eq:markov.inequality.calc})  this implies
\[ \P(y f_x(\Wt t) < 0) \leq 2 |\ell'(0)|^{-1} \alpha^{-1} \xi.\]
To complete the proof, we want to bound $\xi$.  Recall from the calculation \eqref{eq:xi.def} that
\[ \xi = \xi(\gamma) = \phi_{v^*}(\gamma) +  \optlin + \gamma^{-1} \E\Big[|\sip{v^*}x| \ind(y\sip{v^*}x<0)\Big].\]
Fix $\rho >0$ to be chosen later.  We can write
\begin{align}\nonumber
\E[|\sip{v^*}x| \ind(y\sip{v^*}x < 0)] &= \E[|\bar v^\top x| \ind(y \bar v^\top x \leq 0,\ |\bar v^\top x| > \rho)] + \E[|\bar v^\top x| \ind(y \bar v^\top x \leq 0,\ |\bar v^\top x| \leq \rho)]\\ 
\nonumber
&\leq \rho \optlin + \int_\rho^\infty \P(|\bar v^\top x| > t) \mathrm dt \\
\nonumber
&\leq \rho \optlin + \int_\rho^\infty \exp(-t / C_m) \mathrm dt \\
\label{eq:xi.bound}
&= \rho \optlin + C_m\exp(-\rho/C_m).
\end{align}
The first inequality comes from Cauchy--Schwarz, the second from truncating, and the last from the definition of $C_m$-sub-exponential.  Taking $\rho = C_m \log(\nicefrac 1 {\opt})$ results in
\[ \E[|\sip{v^*}x| \ind(y \bar v^\top x \leq 0)]\leq  C_m \optlin \log(\nicefrac 1{\optlin}) + C_m \optlin .\]
Substituting the above into \eqref{eq:sgd.guarantee}, we get
\[ \P(y f_x(\Wt t) < 0) \leq 2 |\ell'(0)|^{-1} \alpha^{-1} \l[ \phi_{v^*}(\gamma) + (1 + \gamma^{-1} C_m) \optlin + \gamma^{-1} C_m \optlin \log(\nicefrac 1{\optlin})\r].\]
\end{proof}

\section{Experiments}\label{sec:experiment}

In this section, we provide some experimental verification of our theoretical results.  We consider a distribution $\calD_{b,\gamma_0}$ that is a mixture of two 2D Gaussians perturbed by both random classification noise and deterministic (adversarial) label noise.  The distribution is constructed as follows.  We first take two independent Gaussians with independent components of unit variance and means $(-3, 0)$ and $(3,0)$, and assign the label $-1$ to the left cluster and $+1$ to the right cluster.  We remove all samples with first component $x_1$ satisfying $|x_1| \leq \gamma_0 = 0.5$, so that we have a hard margin distribution with margin $\gamma_0$.  We then introduce a boundary factor $b>\gamma_0$, and for samples with first component satisfying $|x_1|\leq b$ we deterministically flip the label to the opposite sign.   Finally, for samples with $|x_1| > b$, we introduce random classification noise at level $10\%$, flipping the labels in those regions with probability $0.1$ each.   The symmetry of the distribution implies that an optimal halfspace is the vector $v^*=(1,0)$. 

\begin{figure*}[h]
     \centering
     \begin{subfigure}[b]{0.48\textwidth}
         \centering
         \includegraphics[width=0.75\textwidth]{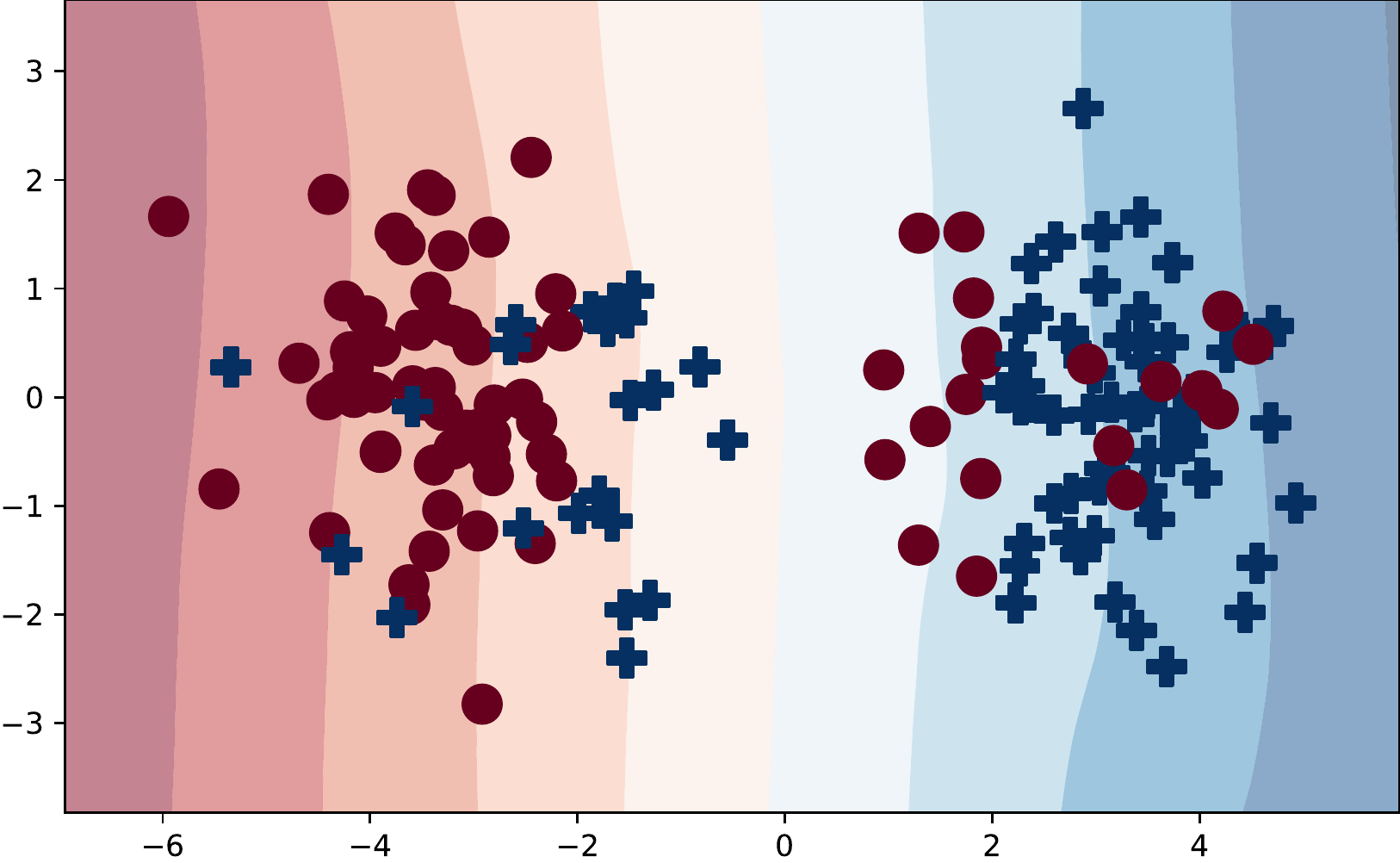}
         \caption{}
         \label{fig:data}
     \end{subfigure}
     \hfill
     \begin{subfigure}[b]{0.48\textwidth}
         \centering
         \includegraphics[width=0.7\textwidth]{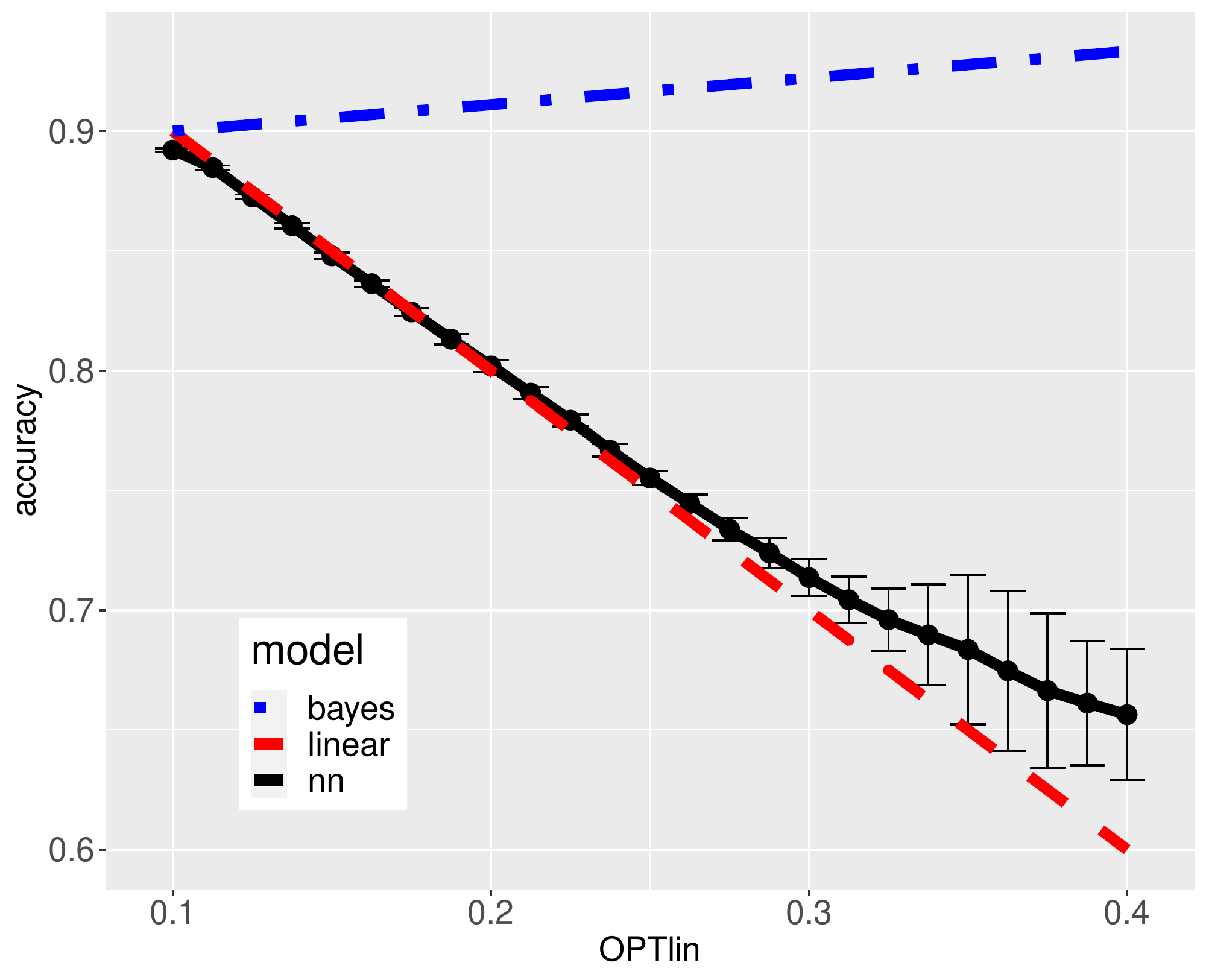}
         \caption{}
         \label{fig:results}
     \end{subfigure}
        \caption{(a) Samples from $\calD_{2.04,0.5}$ with random classification noise of 10\% on $\{|x_1|>2.04\}$ with the boundary term $b=2.04$ chosen so that $\optlin = 0.25$. Blue plus signs correspond to $y=+1$ and red circles to $y=-1$.  The contour plot displays the class probability for the output of a leaky ReLU network trained by online SGD and has dark hues when the neural network is more confident in its predictions.   (b) Test classification accuracy for data coming $\calD_{b,0.5}$.  The red dashed line is the accuracy of the best linear classifier, and the black solid line is the average accuracy of the neural network with error bars over ten random initializations of the first layer weights (experimental details can be found in Appendix \ref{appendix:experiments}). The blue dash-dotted line is the Bayes optimal classifier accuracy.} 
        \label{fig:both}
\end{figure*}

The boundary factor $b$ can be tweaked to incorporate more deterministic label noise which will affect the best linear classifier: if $b$ is larger, $\optlin$ is larger as well.  We give details on the precise relationship of $b$ and $\optlin$ in Appendix~\ref{appendix:experiments}.  But because this `noise' is deterministic, the best classifier over $\calD_{b,\gamma_0}$ (the Bayes optimal classifier) can always achieve accuracy of at least 90\% by using the decision rule
\begin{equation} \label{eq:bayes.rcn}
y_{\mathrm{Bayes}} = \begin{cases} +1, &x_1\in (-b, 0) \cup (b, \infty),\\
-1, & x_1 \in (-\infty, b] \cup [0, b].\end{cases} 
\end{equation}
Since the error for the Bayes decision rule corresponds to the region $\{|x_1| > b\}$ with random classification noise, we can exactly calculate the error for the Bayes classifier as well as $\optlin$.  As $b$ increases, the region with random classification noise becomes smaller, and thus the Bayes classifier gets better as the linear classifier becomes worse on $\calD_{b,\gamma_0}$.  This makes $\calD_{b,\gamma_0}$ a good candidate for understanding the performance of SGD-trained one-hidden-layer networks in comparison to linear classifiers.  Further, to our knowledge no previous work has been able to show that neural networks can provably generalize if the data distribution is $\calD_{b,\gamma_0}$.\footnote{There are two reasons that no other work can show generalization bounds in the settings we consider.  The first is the presence of adversarial label noise.  The second is that our generalization bound holds for neural networks with finite width and any initialization.  All previous works fail to allow at least one of these conditions.}

Since $\calD_{b,\gamma_0}$ is a subexponential hard margin distribution, Corollary \ref{corollary:hard.margin} shows that we can expect an SGD-trained leaky ReLU network on $\calD_{b,0.5}$ to achieve a test set accuracy of at least $1 - C\cdot \optlin \log(\nicefrac 1\optlin)$ for some constant $C \geq 1$.  We ran experiments on such a neural network with $m=1000$ neurons and learning rate $\eta=0.01$ and first layer weights initialized as independent normal random variables with variance $1/m$ (see Appendix \ref{appendix:experiments} for more details on the experiment setup).    In Figure \ref{fig:data} we plot the decision boundary for the SGD-trained neural network on the distribution $\calD_{2.04, 0.5}$, where $b=2.04$ is chosen so that $\optlin = 0.25$.  We notice that the decision boundary is almost exactly linear and is essentially the same as that of the best linear classifier $(x_1, x_2)\mapsto \sgn(x_1)$.  And in Figure \ref{fig:results}, we see that the neural network accuracy is almost exactly equal to $1-\optlin$ when $\optlin\leq 0.30$ and that the network slightly outperforms the best linear classifier when $\optlin>0.30$.  

In Appendix \ref{appendix:experiments} we conduct additional experiments to better understand whether this behavior is consistent across hyperparameter and architectural modifications to the network.  When using the bias-free networks of the form \eqref{eq:leaky.relu.def} we consider in this paper, we found that one-hidden-layer SGD-trained networks failed to generalize better than a linear classifier when using $\tanh$ activations (Figure~\ref{fig:tanh}), using different learning rates (Figure~\ref{fig:lr}), different initialization variances (Figure~\ref{fig:init.sd}), and using multiple-pass SGD rather than online SGD (Figure~\ref{fig:batch}).  On the other hand, we found that introducing bias terms can lead to decision boundaries closer to that of the Bayes-optimal classifier (Figure~\ref{fig:secondandbias}).  Interestingly, this behavior was strongly dependent on the initialization scheme used: when using an initialization variance of $1/m^4$, a linear decision boundary was consistently learned, while using an initialization variance of $1/m$ lead to approximately Bayes-optimal decision boundaries.  By contrast, the result we present in Theorem \ref{thm:online.sgd.leaky.relu} holds for \textit{arbitrary} initialization schemes.  This suggests that a new analytical approach would be needed in order to guarantee neural network generalization performance better than that of a linear classifier on $\calD_{\gamma_0,b}$.

\section{Discussion}
We have shown that overparameterized one-hidden-layer networks can generalize almost as well as the best linear classifier over the distribution for a broad class of distributions.  Our results imply two related but distinct insights on SGD-trained neural networks.  First, regardless of the initialization scheme and number of neurons, SGD training will produce neural networks that are competitive with the best linear predictor over the data, providing theoretical support for the hypothesis presented by~\citet{nakkiran2019sgd} that the performance of SGD-trained networks in the early stages of training can be explained by that of a linear classifier.  Second, a linearly separable dataset can be corrupted by adversarial label noise and overparameterized neural networks will still be able to generalize, despite the capacity to overfit to the label noise.

A number of extensions and open questions remain.  First, our analysis was specific to one-hidden-layer networks with the leaky-ReLU activation.  We are interested in extending our results to more general neural network architectures.  Second, a natural question is whether or not there are concept classes that are more expressive than halfspaces for which overparameterized neural networks can generalize for noisy data.   We are particularly keen on understanding this question for finite width neural networks that are not well-approximated by the NTK.  

\section*{Acknowledgements}
We thank James-Michael Leahy for a number of helpful discussions.  We thank Maria-Florina Balcan for pointing us to a number of works on learning halfspaces in the presence of noise.  

\appendix

\section{Proof of Lemma \ref{lemma:key.identity}}\label{appendix:key.identity.proof}
In this section we will prove a stronger version of Lemma \ref{lemma:key.identity} that holds for any increasing activation.
\begin{lemma}\label{lemma:general.key.identity}
Suppose that $\sigma$ is non-decreasing.  For $V$ defined in \eqref{eq:vdef}, for any $(x,y)\in \R^d \times \{ \pm 1\}$, for any $W\in \R^{m\times d}$, and any $\gamma\in (0,1)$:
\begin{align}\label{eq:key.identity}
   \nonumber &y\sip{\nabla f_x(W)}{V} \\
   &\geq  a \gamma m^{-1/2}  \Big[ 1 -  \ind(y \sip{v^*}x\in [0,\gamma)) - ( 1+\gamma^{-1}) |\sip{v^*}x| \ind(y\sip{v^*}x < 0)\Big] \summ j {m} \sigma'(\sip{w_j}x).
\end{align}
\end{lemma}
For $\sigma(z) = \max(\alpha z, z)$, we have $\summ j m \sigma'(\sip{w_j}x) \in [\alpha m, m]$, and hence the above implies Lemma~\ref{lemma:key.identity}:
\begin{align*}\nonumber &y\sip{\nabla f_x(W)}{V} \\
   &\geq  a \gamma m^{-1/2}  \Big[ \alpha m -  m\ind(y \sip{v^*}x\in [0,\gamma)) - m( 1+\gamma^{-1}) |\sip{v^*}x| \ind(y\sip{v^*}x < 0)\Big] \\
   &= a \gamma \sqrt m  \Big[ \alpha  -  \ind(y \sip{v^*}x\in [0,\gamma)) - ( 1+\gamma^{-1}) |\sip{v^*}x| \ind(y\sip{v^*}x < 0)\Big].
   \end{align*}
\begin{proof}[Proof of Lemma \ref{lemma:general.key.identity}]
By the definition of $V$ (see \eqref{eq:vdef}), we have
\begin{align*}
    y\sip{\nabla f_x(W)}V &= \summ j m a_j \sigma'(\sip{w_j}x) \sip{yv_j}x\\
    &= a m^{-1/2} \summ j m \sigma'(\sip{w_j}x) \sip{yv^*}x \\
    &= a m^{-1/2} \summ j {m} \sigma'(\sip{w_j}x) \sip{yv^*}x \Big[ \ind( y \sip {v^*}x \geq \gamma) + \ind(y \sip{v^*}x \in [0,\gamma)) + \ind(y \sip{v^*} x < 0)\Big].
    \end{align*}
The second line uses that $a_j v_j = |a_j| v^* = a v^*$.  Continuing, we have    
    
    \begin{align*}
    &y\sip{\nabla f_x(W)}V \\
    &\geq a \gamma  m^{-1/2} \ind(y \sip{v^*}x \geq \gamma) \cdot \summ j {m} \sigma'(\sip{w_j}x)\\
    &\quad + a  m^{-1/2} \summ j {m} \sigma'(\sip{w_j}x)\sip{yv^*}x \Big [\ind(y \sip{v^*}x \in [0,\gamma)) + \ind (y \sip{v^*}x < 0)\Big] \\
    &\geq a \gamma  m^{-1/2} \ind(y \sip{v^*}x \geq \gamma) \cdot \summ j {m} \sigma'(\sip{w_j}x) + a m^{-1/2}  \summ j {m} \sigma'(\sip{w_j}x)\sip{yv^*}x\ind (y \sip{v^*}x < 0) \\
    \nonumber
    &= a  \gamma  m^{-1/2} [1 - \ind(y \sip{v^*}x \in [0, \gamma))- \ind(y \sip{v^*}x < 0)] \summ j {m} \sigma'(\sip{w_j}x) \\
    &\quad + a  m^{-1/2} \summ j {m} \sigma'(\sip{w_j}x)\sip{yv^*}x\ind (y \sip{v^*}x < 0) \\
    \nonumber
    &\geq a \gamma  m^{-1/2} [1 - \ind(y \sip{v^*}x \in [0, \gamma)) - \ind(y \sip{v^*}x < 0)] \summ j {m} \sigma'(\sip{w_j}x) \\
    &\quad - a  m^{-1/2} |\sip{v^*}x| \ind(y \sip{v^*}x < 0) \summ j {m} \sigma'(\sip{w_j}x)\\
    &= a m^{-1/2} \Big[ \gamma - \gamma \ind(y \sip{v^*}x\in [0,\gamma)) - (\gamma + 1) |\sip{v^*}x| \ind(y\sip{v^*}x < 0)\Big] \summ j {m} \sigma'(\sip{w_j}x).
\end{align*}
The first and second inequalities use that $\sigma'(z)\geq 0$ and that $a>0$.  The third inequality uses that $x \geq -|x|$.  This proves \eqref{eq:key.identity}.
\end{proof}

\section{Additional Experiments and Experiment Details}\label{appendix:experiments}
In this section, we give details on the experiments given in Section \ref{sec:experiment}.  Let us first describe how we calculate $\optlin$ for $\calD_{b,\gamma_0}$.  To remind the reader, we begin by constructing $\calD_{b,\gamma_0}$ with a mixture of two independent Gaussians centered at $(-3, 0)$ and $(3,0)$ with independent unit variance components and then remove all data that has $x_1$ component in the interval $[-\gamma_0, \gamma_0]$.   We assign initial labels to be $-1$ if $x_1<0$ and $1$ if $x_1>0$.  For boundary factor $b>\gamma_0$, the deterministic adversarial label noise then assigns the label $1$ if $-b < x_1 < -\gamma_0$, and assigns the label $-1$ if $\gamma_0<x_1<b$. The final labels are determined by flipping labels for samples with $|x_1|>b$ with probability $p$ each. 

By construction, an optimal unit-norm halfspace classifier is given by the vector $(1,0)$, and this classifier is a hard-margin classifier with margin $\gamma_0>0$. The optimal halfspace classification error is given as the sum of two terms: (1) the random classification noise for the region $|x_1|>b$, and (2) the deterministic noise in the region $|x_1|<b$.   The error introduced from the deterministic, adversarial noise is the proportion of  2D Gaussian that has $x_1$ coordinate lying between $3-\gamma_0$ and $3-b$, conditioned on the fact that $x_1$ is at most $3-\gamma_0$.  We can directly calculate this as
\[ \mathrm{err}_{\mathrm{det}} = \frac{\mathbb{P}(3-b < N(0,1) \leq 3-\gamma_0) }{\mathbb{P}(N(0,1) \leq 3 - \gamma_0)} = \frac{\Phi(3-\gamma_0) - \Phi(3-b)}{\Phi(3-\gamma_0)},\]
where $\Phi$ is the standard normal cumulative distribution function.  Similarly, the error for the best linear classifier introduced by the random classification noise at rate $p$ is given by $p$ times the proportion of a 2D Gaussian that has $x_1$ coordinate smaller than $3-b$, conditioned on the $x_1$ coordinate being at most $3-\gamma_0$.  That is,
\begin{equation}\label{eq:err.rcn}
\mathrm{err}_{\mathrm{rcn}} = p \frac{\mathbb{P}(N(0,1) \leq 3-b)}{\mathbb{P}(N(0,1) \leq 3-\gamma_0)} = p\frac{\Phi(3-b)}{\Phi(3-\gamma_0)}.
\end{equation}
The total error for the optimal linear classifier is then given by
\begin{align*}\optlin &= \mathrm{err}_{\mathrm{det}} + \mathrm{err}_{\mathrm{rcn}} \\
&= \frac{1}{\Phi(3-\gamma_0)} \Big( \Phi(3 - \gamma_0) - \Phi(3-b) + p \Phi(3-b) \Big)\\
&= \frac{1}{\Phi(3-\gamma_0)} \Big( \Phi(3-\gamma_0) - (1-p) \Phi(3-b)\Big).\end{align*}
Solving for the boundary term in terms of $\optlin$ results in
\[ b = 3 - \Phi^{-1}\Bigg(\frac{1-\optlin}{1-p}  \Phi(3 - \gamma_0)\Big) \Bigg).\]
We then consider $\optlin$ in a grid and take the corresponding values of the boundary term $b$ to produce a distribution with hard margin $\gamma_0=0.5$ where the best population risk achieved by a linear classifier is $\optlin$.  We note that the Bayes-optimal classifier has decision rule given by \eqref{eq:bayes.rcn} with the Bayes risk equal to $\mathrm{err}_{rcn}$.  

The baseline neural network model we use, and the neural network used for Figure \ref{fig:results}, is as follows. We use a bias-free one-hidden-layer network~\eqref{eq:leaky.relu.def} with leaky ReLU activations (with $\alpha = 0.1)$ and $m=1000$ neurons with outer layer fixed at initialization with half of the $a_j$ equal to $+1/\sqrt{m}$ and the other half equal to $-1/\sqrt{m}$.  We initialize the hidden layer weights independently with normal random variables with variance $1/m$, so that $G_0^2 = \|\Wt 0\|_F^2 = O(1)$ with high probability (ignoring $d=2$ as a small constant).   We use online SGD (i.e. batch size one with a new sample used at each iteration) with $T =$ 20,000 samples\footnote{In ablation studies with $T=$ 100,000 samples, we observed no discernible difference in the classification accuracy, unless otherwise stated.} trained on the cross-entropy loss with fixed learning rate $\eta=0.01$.   We use a validation set of size 10,000 and evaluate performance on the validation set every 100 SGD iterations, and we take the model with the smallest validation error over the $T$ samples and evaluate its performance on a fresh test set (sampled independently from the training and validation sets) of 100,000 samples to produce the final test set accuracy.   We then repeat this experiment ten times for each level of $\optlin$ considered with the ten trials using different seeds for both the initialization of the first layer weights and for the sequence of data observed in online SGD (i.e. for fixed data $\{x_t\}_1^T$, we use a permutation $\pi :[T]\to [T]$ to permute the data $\{x_t\}_1^T\mapsto \{x_{\pi(t)}\}_1^T$).  We plot the average across the ten trials with error bars corresponding to one standard deviation in Figure \ref{fig:results}; in all subsequent modifications to this baseline neural network model, we will always plot the mean and error bars over the ten trials considered.  We calculate the Bayes-optimal classification error by using the boundary term corresponding to each value of $\optlin$ and plotting $\mathrm{err}_{\mathrm{rcn}}$ as the blue dash-dotted line in Figure \ref{fig:results}.   Code for our experiments is available on Github.\footnote{\href{https://github.com/spencerfrei/nn_generalization_agnostic_noise}{https://github.com/spencerfrei/nn\_generalization\_agnostic\_noise}}

In Figure \ref{fig:baseline.decision.boundary}, we show the decision boundary of the baseline neural network for $\optlin\in \{0.1, 0.25, 0.40\}$ for four independent initializations of the first layer weights.  For each level of $\optlin$, the neural network classifier has a nearly linear decision boundary.
\begin{figure}[h]
    \centering
    \includegraphics[width=\textwidth]{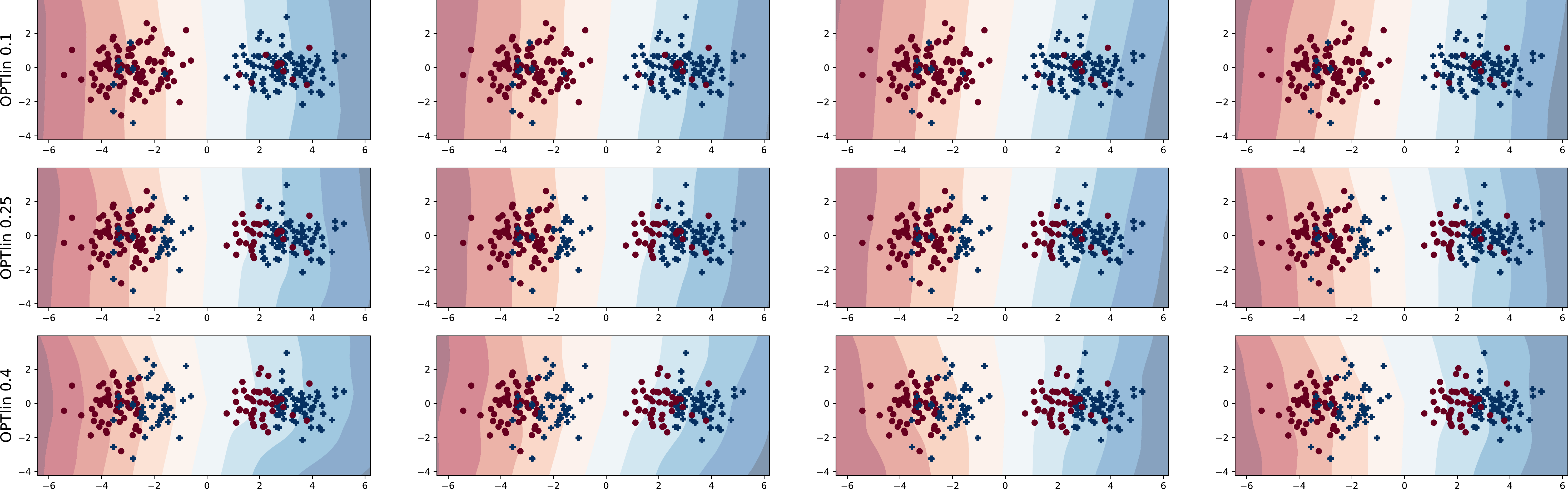}
    \caption{Decision boundary of an SGD-trained neural network on $\calD_{b,\gamma_0}$, where $b$ is chosen so that $\optlin \in \{0.1, 0.25, 0.40\}$, across four different random initializations.  The decision boundary is the line where the region changes from light red to light blue, and the dark regions are areas where the neural network classifier has the highest confidence.  Even in the presence of substantial, adversarial noise, the decision boundary is close to linear.} 
    \label{fig:baseline.decision.boundary}
\end{figure}

In Figure \ref{fig:tanh}, we modify the baseline neural network by having $\tanh$ activations instead of leaky ReLU.  Although $\tanh$ is highly nonlinear, the performance of $\tanh$ networks is essentially the same as the leaky ReLU network, and the decision boundaries are approximately linear even for large $\optlin$.  

\begin{figure}[h]
     \centering
     \begin{subfigure}[b]{0.68\textwidth}
         \centering
         \includegraphics[width=\textwidth]{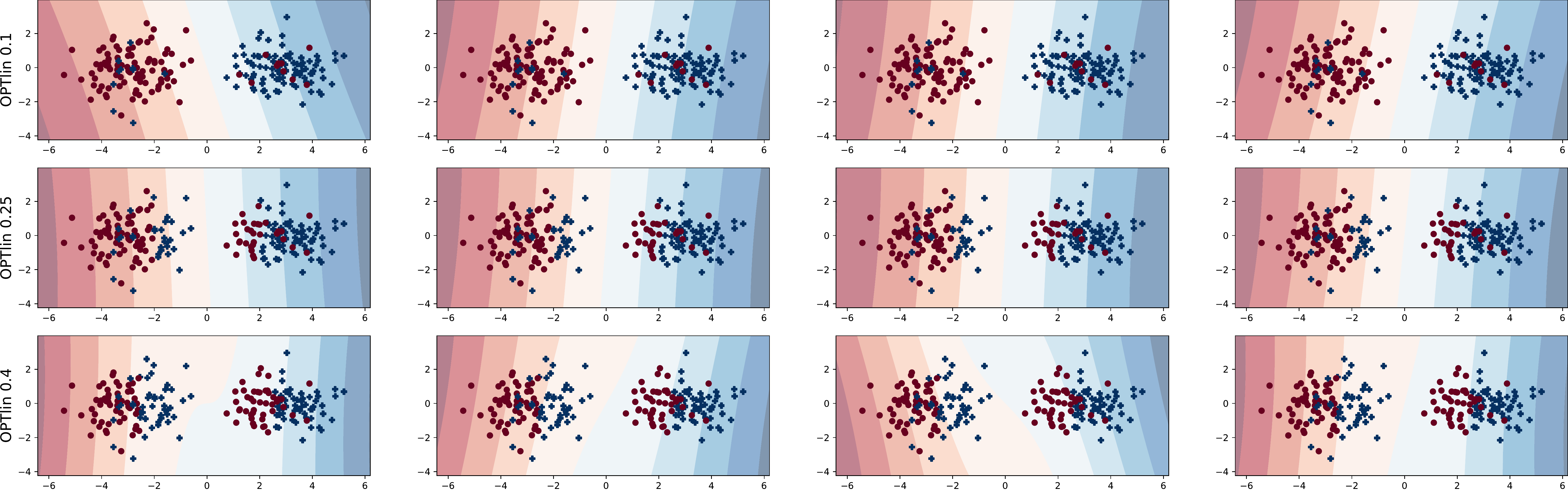}
         \caption{}
         \label{fig:tanh.decision}
     \end{subfigure}
     \hfill
     \begin{subfigure}[b]{0.3\textwidth}
         \centering
         \includegraphics[width=\textwidth]{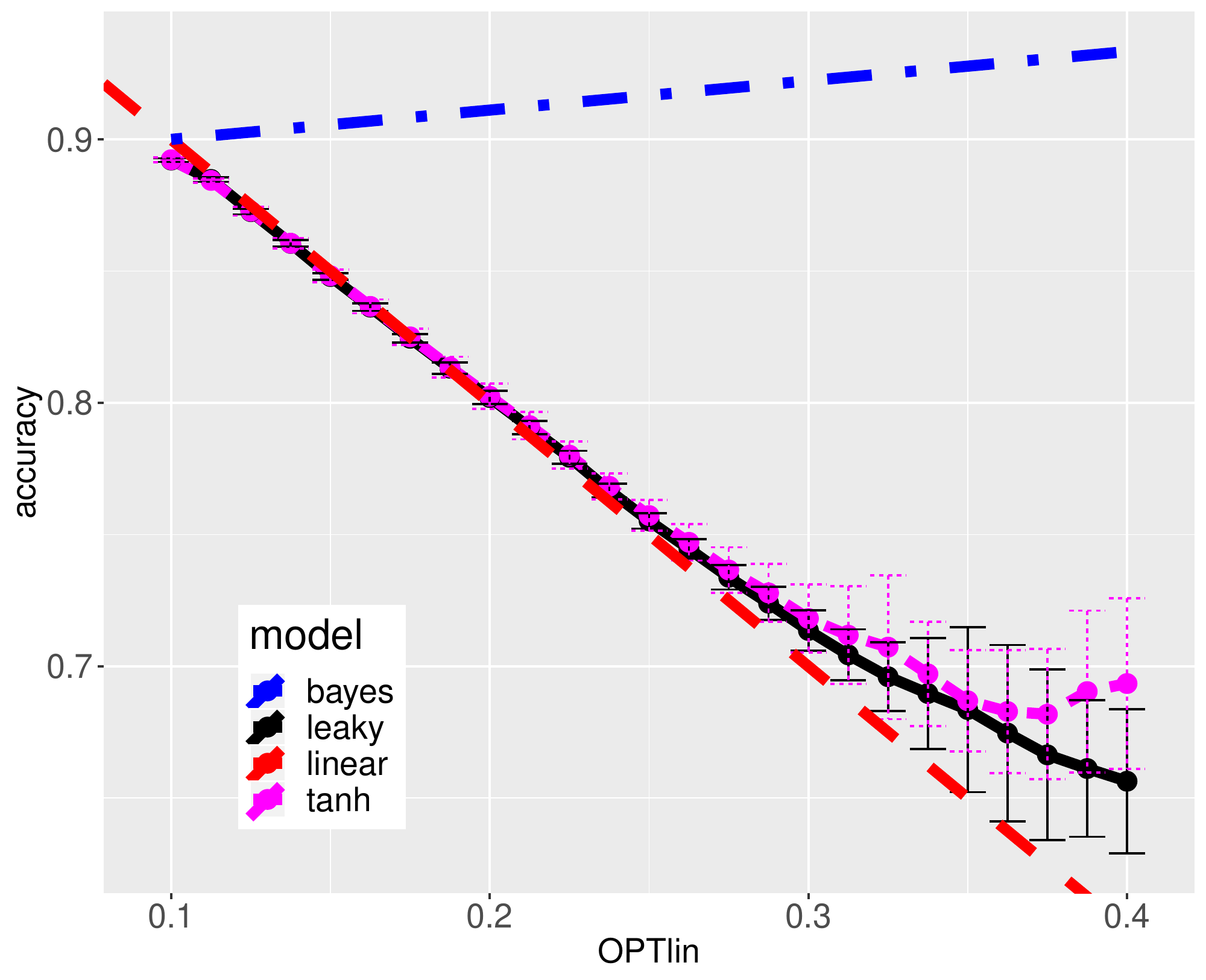}
         \caption{}
         \label{fig:tanh.results}
     \end{subfigure}
        \caption{(a)  Decision boundary for the same setup as the baseline neural network except with $\tanh$ activations. Columns correspond to different random initializations.  Compare with Figure \ref{fig:baseline.decision.boundary}.  Even for nonlinear activations we still see an almost perfectly linear decision boundary for $\optlin = 0.25$.   (b) Test classification accuracy.  The performance of leaky ReLU and tanh networks are almost exactly the same and match the performance of the best linear predictor until extreme levels of noise.} \label{fig:tanh}
\end{figure}

In Figure \ref{fig:lr}, we consider variations of the learning rate from the baseline $\eta=0.01$ to $\eta \in \{0.1, 0.001\}$.  Overall, the test accuracy is essentially the same, albeit of smaller variance across initializations when the learning rate is smaller.   When the learning rate is smaller, the decision boundary is almost perfectly linear, even when $\optlin=0.4$.  When $\eta=0.1$, the decision boundary changes significantly for different initializations of the first layer weights, resulting in a higher variance for the test accuracy, but the decision boundary is still a rough perturbation of the best linear classifier decision boundary.   

In Figure \ref{fig:init.sd}, we examine the effect of modifying the initialization of the first layer weights from the baseline variance of $1/m$ to $\mathrm{Var}(w_{i,j}^{(0)}) \in \{m^{-2}, 1\}$.   The overall accuracy is essentially the same across initialization variances.  The decision boundary becomes more smooth and linear when the variance is smaller. When the variance is larger, the decision boundary is more disjointed and nonsmooth, but is still roughly a perturbation of the best linear classifier decision boundary.

In Figure \ref{fig:batch}, we consider the modification of using 100 epochs of multiple-pass SGD with batch size 32.  All other architectural and optimization hyperparameters from the baseline case are the same.  We see that the decision boundary and test accuracy has less variance across random initializations, which we interpret as being due to the averaging effect of increasing the batch size from 1 to 32.  The test classification accuracy is virtually indistinguishable from the online SGD case.

In Figure \ref{fig:secondandbias}, we consider two modifications to the neural network: (1) increasing the width from the baseline of $m=10^3$ to $m=10^5$, and (2) introducing trainable bias terms and training the second layer weights.  The difference in (1) is imperceptible and so we do not plot the decision boundary in this case.   On the other hand, we observed that with trainable biases and second layer weights, the neural network can come close to Bayes-optimal classifier accuracy provided the initialization variance is chosen appropriately.  In particular, with an initialization variance of $1/m$, the network is able to learn a nonlinear decision boundary, but with an initialization variance of $1/m^4$, the network only learns a linear decision boundary.\footnote{For the experiments involving trainable biases and second layer weights, we increased the sample size from $T=$ 20,000 to $T=$ 100,000 since the validation accuracy was still continuing to increase with $T=$ 20,000 for the initialization variance of $1/m$.  This was the only set of experiments where we noticed such behavior.}  We note that our result in Theorem \ref{thm:online.sgd.leaky.relu} holds for \textit{any} initialization, and thus these experiments suggest that we would need to introduce new analyses in order to get generalization performance much better than a linear classifier.  Additionally, these experiments suggest that the ability of an SGD-trained network to generalize better than a linear classifier on $\calD_{\gamma_0,b}$ is strongly dependent upon the initialization scheme used and the usage of bias terms.

As a final study on $\calD_{b,\gamma_0}$, we consider a three-hidden-layer fully connected network of the form
\begin{equation}
    x^{(1)} = \sigma(W^{(1)}x),\quad x^{(l)} = \sigma(W^{(l)} x^{(l-1)} ),\, l=2,3,\quad f_x(\vec W, \vec b) = a^\top x^{(3)},\label{eq:4layer}
\end{equation}
where $W^{(1)}\in \R^{m\times d}$, $W^{(l)}\in \R^{m\times m}$ for $l=2, 3$,  $a\in \R^{m\times 1}$ are all trainable weights, and $\sigma$ is again the leaky ReLU with $\alpha=0.1$.  In Figure \ref{fig:4layer}, we plot the decision boundary and accuracy for this four layer network (with $m=100$) with each layer's weights initialized with variance $1/m$ and the final layer weights initialized at $\pm 1/\sqrt{m}$ and the same learning rate of $0.01$.  This network is able to learn a better partition of the input space and is able to generalize almost as well as the Bayes optimal classifier, enjoying the same trend of increase in performance as $\optlin$ increases that holds for the Bayes optimal classifier.   This experiment suggests that although there is evidence that bias-free one-hidden-layer networks fail to learn $\calD_{b,\gamma_0}$ up to an accuracy better than a linear classifier, bias-free networks with multiple hidden layers can.

We also conducted a series of experiments to emphasize that although it seems that bias-free SGD-trained one-hidden-layer networks cannot learn $\calD_{b,\gamma_0}$ to an accuracy better than a linear classifier, there are simple distributions for which such networks easily outperform linear predictors.  We construct a distribution $\tilde \calD_b$ as follows.  We introduce a boundary factor $b>0$ and sample an isotropic 2D Gaussian, and then assign the label $+1$ if $x_2 < b |x_1|$, and the label $-1$ otherwise.  Every (bias-free) halfspace for the marginal distribution of a 2D Gaussian partitions any circle centered at the origin into two equal-sized halves.   By symmetry of the isotropic Gaussian, this means the best halfspace will have error exactly equal to the proportion of $+1$ lying in the region with $0 < x_2 \leq b|x_1|$.    If we denote the angle corresponding to the region $\{x_2 \geq b|x_1|\}$ where $y=-1$ as $2\theta$, then this means the error of the best linear classifier is given by $\optlin = \frac{\pi - 2\theta}{2\pi} = \nicefrac 12 - \nicefrac \theta \pi$ (see Figure \ref{fig:abs.distribution}).  The angle $\theta \in [0,\nicefrac \pi 2]$ is given by $\theta = \arctan(1/b)$, and thus we can solve for $\optlin$ in terms of $b$.  When $b \to 0$, the error for the best halfspace converges to $0$, while as $b\to \infty$ we have $\optlin \to \nicefrac 12$.  The Bayes classifier achieves accuracy $100\%$ with the decision rule $y_{\mathrm{Bayes}}=1$ if $x_2 < b |x_1|$ and $-1$ otherwise.  

The 2D Gaussian satisfies $1$-anti-concentration and Corollary \ref{corollary:anti.concentration} guarantees that an SGD-trained neural network will achieve a test set accuracy of at least $1-\tilde \Omega(\sqrt{\optlin})$.  We see in Figure \ref{fig:abs} that the neural network performs quite a bit better than the best linear classifier (and significantly better than $1-\sqrt{\optlin}$), with the decision boundary notably nonlinear and attuned to the distribution of the data.  In summary, one-hidden-layer bias-free leaky ReLU networks trained by SGD can learn nonlinear decision boundaries, but apparently not the type of decision boundary necessary to outperform linear classifiers on $\calD_{b,\gamma_0}$. 

\begin{figure}[ht!]
     \centering
     
     \begin{subfigure}[b]{\textwidth}
         \centering
         \includegraphics[width=0.5\textwidth]{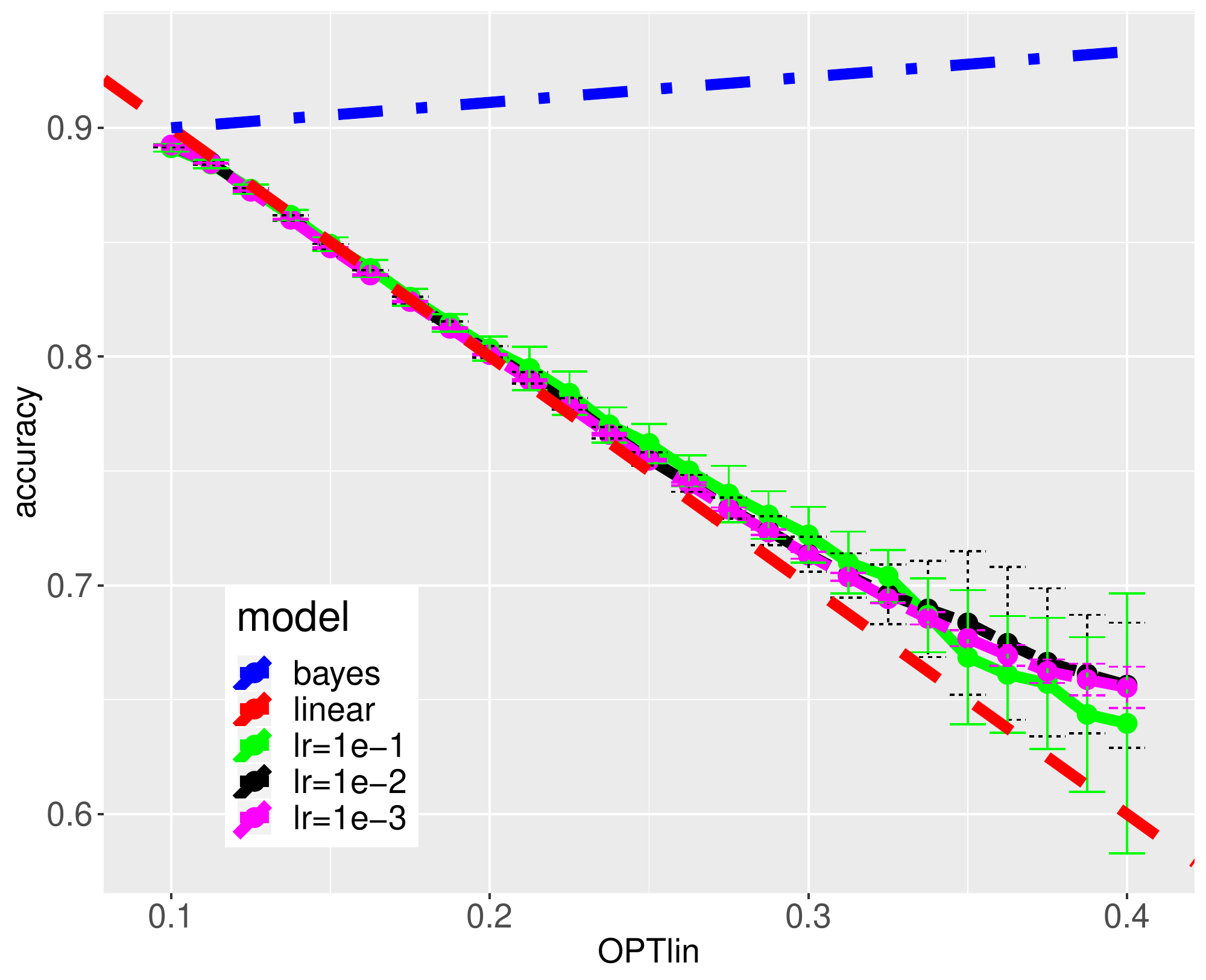}
         \caption{}
         \label{fig:lr0.001.results}
     \end{subfigure}
     \begin{subfigure}[b]{\textwidth}
         \centering
         \includegraphics[width=0.9\textwidth]{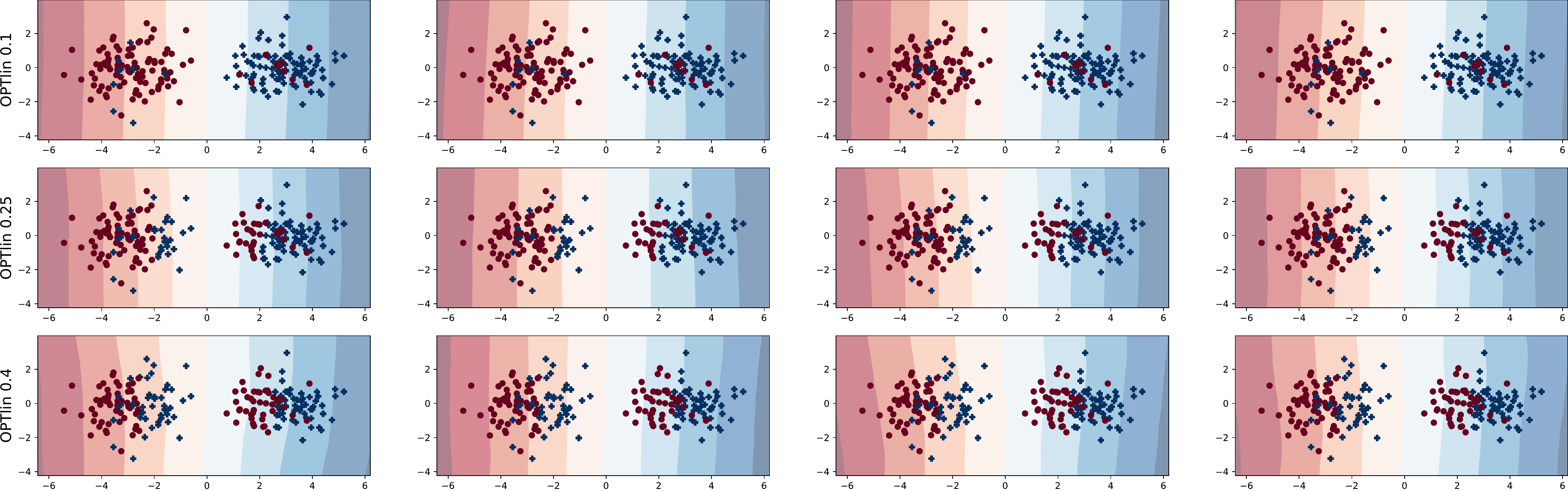}
         \caption{}
         \label{fig:lr0.001.decision.boundary}
     \end{subfigure}
     \begin{subfigure}[b]{\textwidth}
         \centering
         \includegraphics[width=0.9\textwidth]{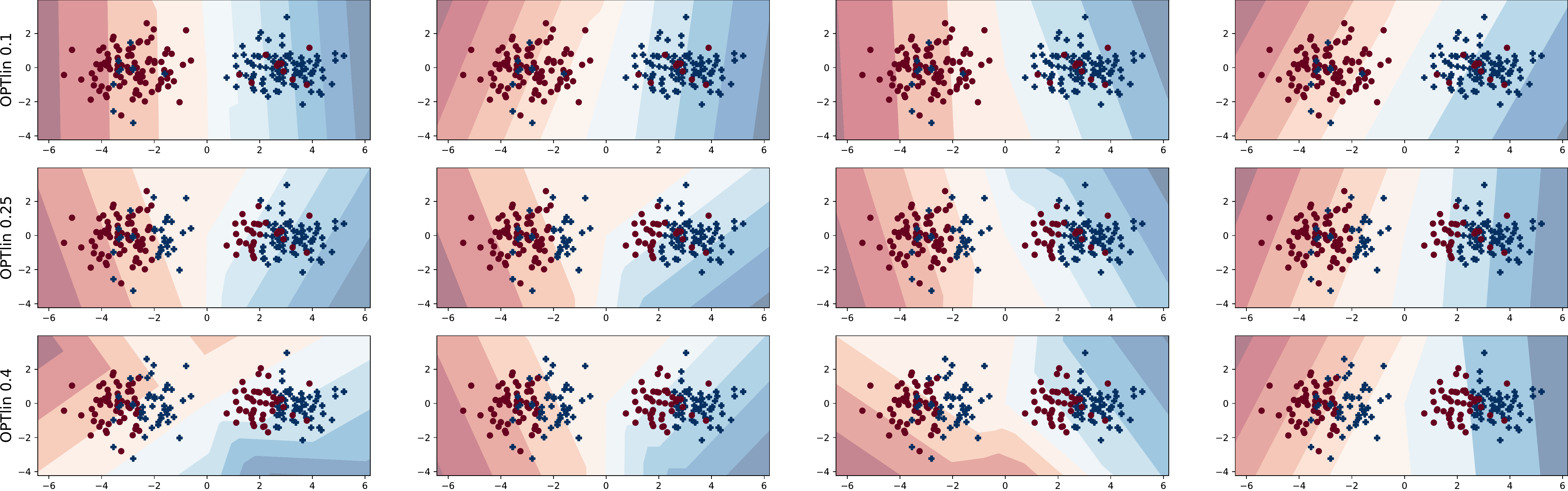}
         \caption{}
         \label{fig:lr0.1.decision.boundary}
     \end{subfigure}
        \caption{(a) Test classification accuracy for learning rates $\eta=0.1$ and $\eta=0.001$ compared to baseline $\eta=0.01$.  Large learning rates lead to a larger variance in performance.  (b) Decision boundary for $\eta = 0.001$ is consistently linear.  (c) Decision boundary for $\eta=0.1$ varies over initializations but is roughly a perturbation of the linear classifier decision boundary.} \label{fig:lr}
\end{figure}

\begin{figure}[ht!]
     \centering
     
     \begin{subfigure}[b]{\textwidth}
         \centering
         \includegraphics[width=0.5\textwidth]{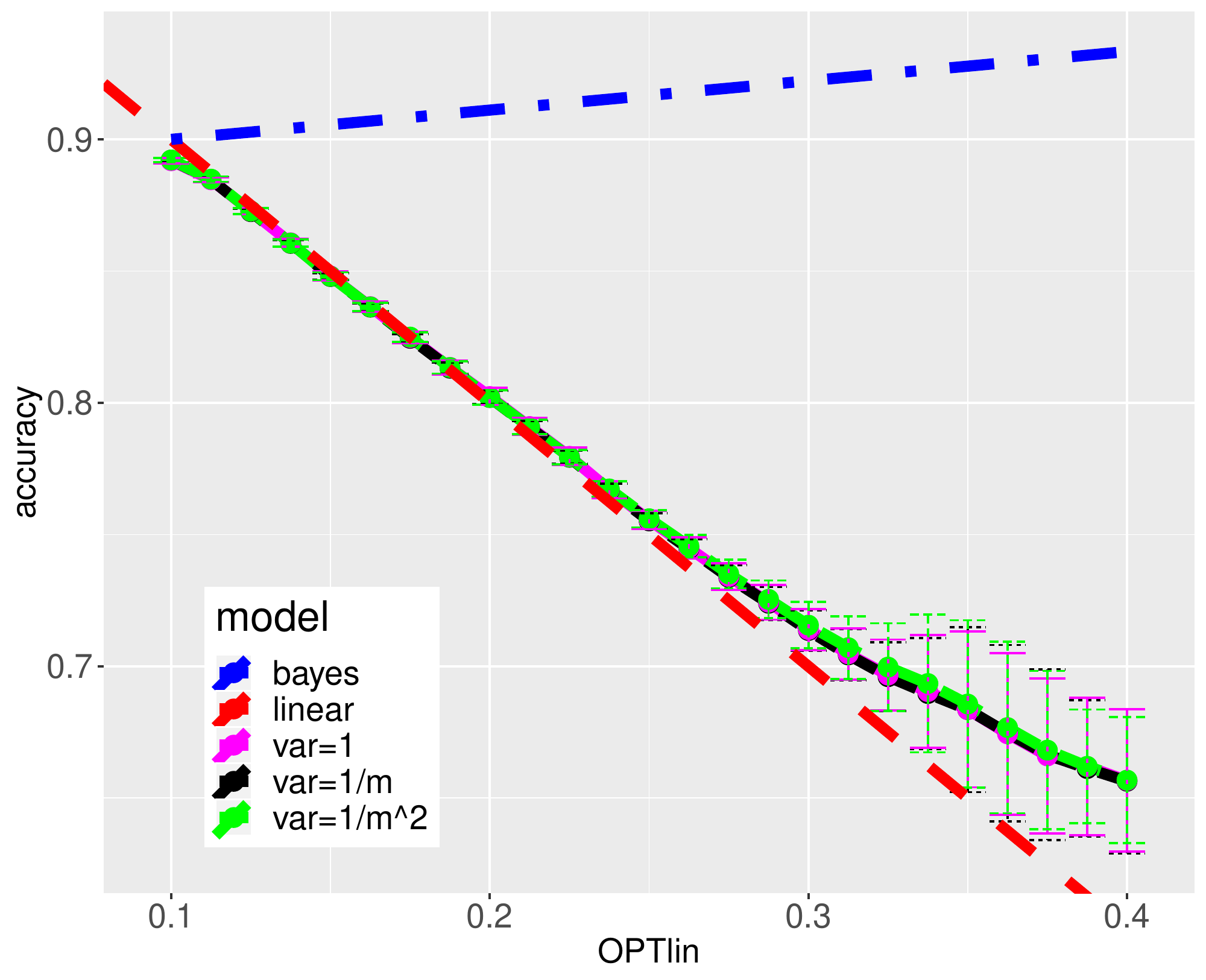}
         \caption{}
         \label{fig:init.sd.results}
     \end{subfigure}
     \begin{subfigure}[b]{\textwidth}
         \centering
         \includegraphics[width=0.9\textwidth]{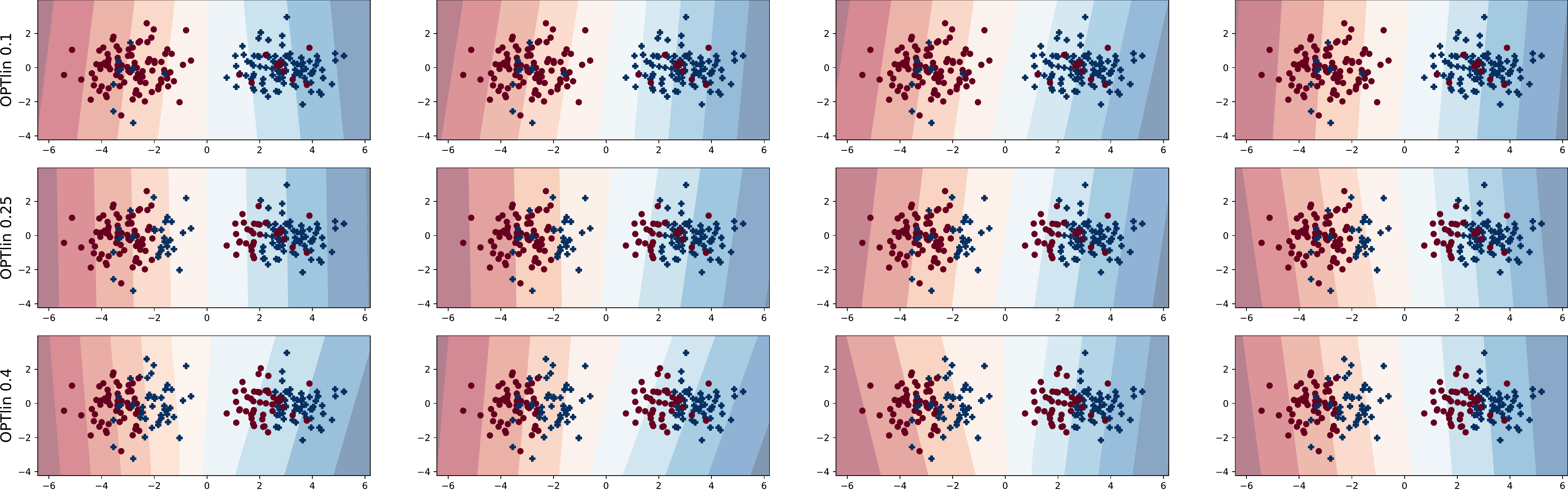}
         \caption{}
         \label{fig:stdm-1.decision.boundary}
     \end{subfigure}
     \begin{subfigure}[b]{\textwidth}
         \centering
         \includegraphics[width=0.9\textwidth]{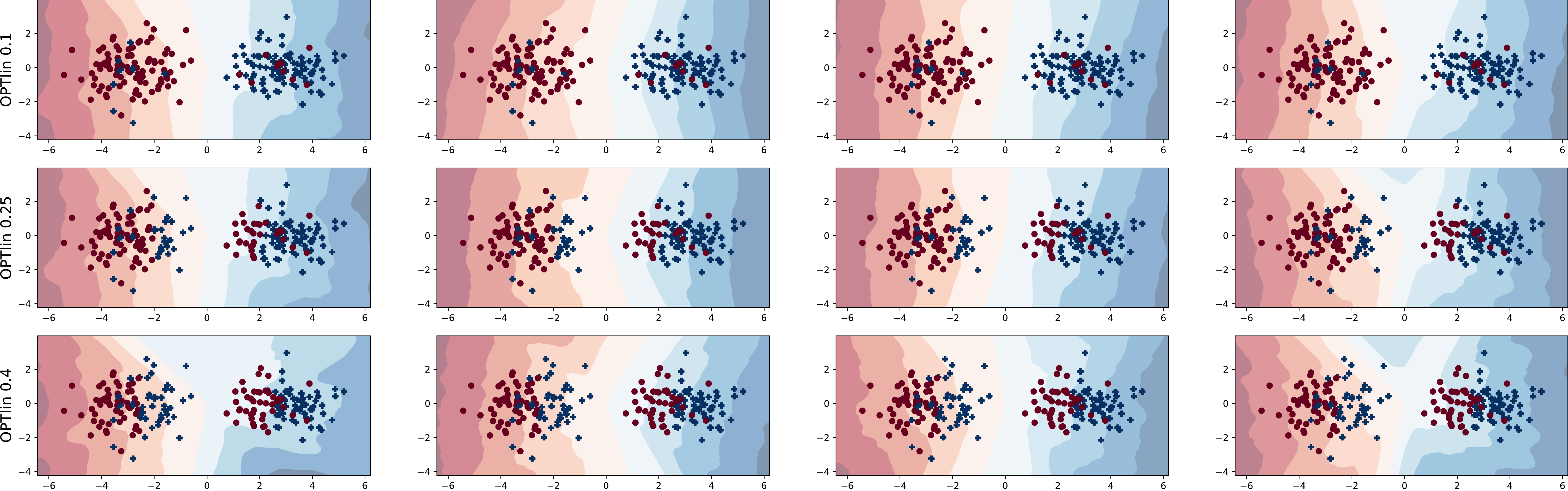}
         \caption{}
         \label{fig:std1.decision.boundary}
     \end{subfigure}
        \caption{(a) Test classification accuracy for different values of the variance of the first layer weight initialization.  The baseline neural network has variance $1/m$.     (b) Decision boundary for the smaller variance $1/m^2$ is more consistently linear.  (c) Decision boundary for variance $1$ has more variation across random initializations, but are roughly perturbations of the linear classifier decision boundary.} \label{fig:init.sd}
\end{figure}

\begin{figure}[h]
     \centering
       \begin{subfigure}[b]{0.5\textwidth}
         \centering
         \includegraphics[width=\textwidth]{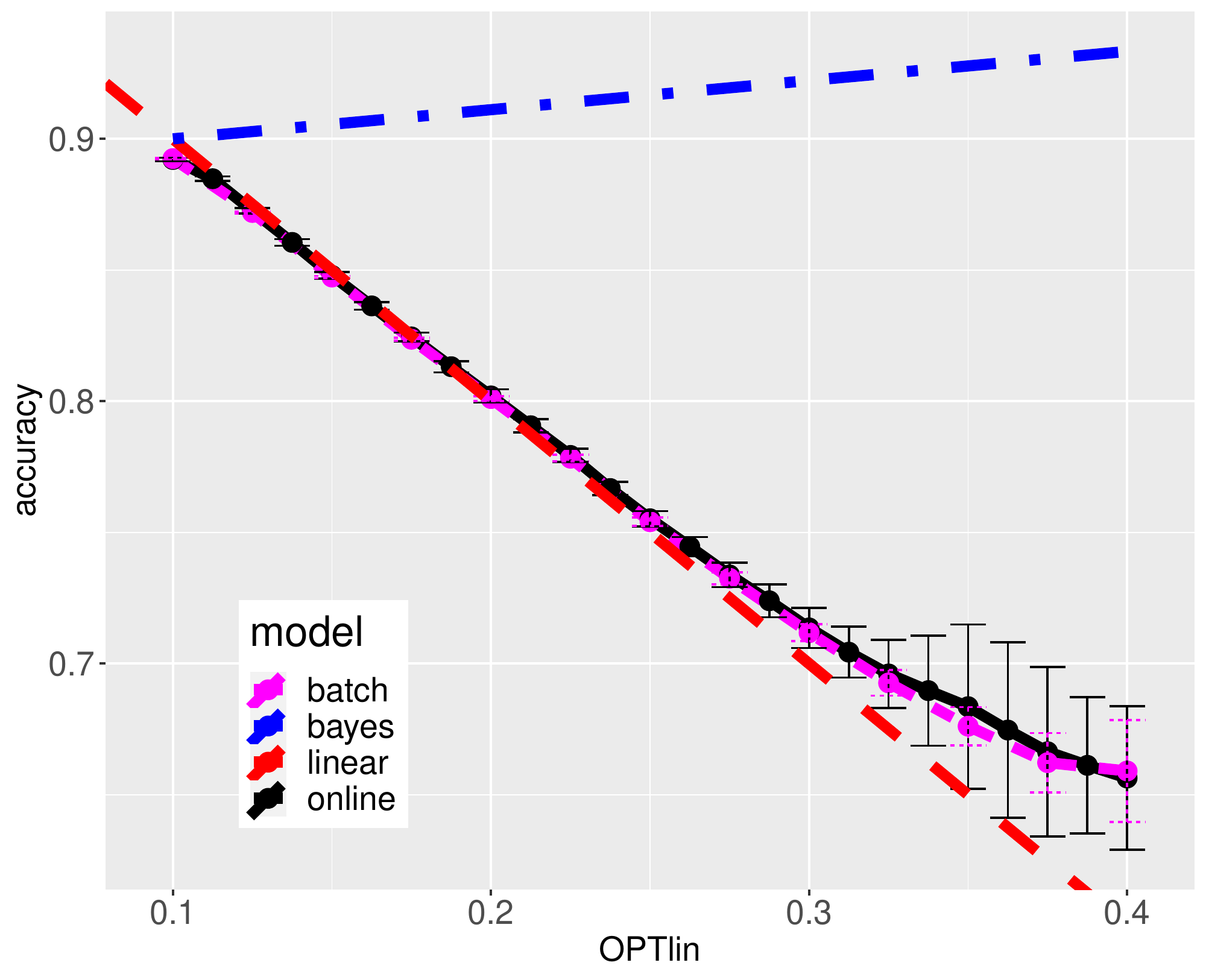}
         \caption{}
     \end{subfigure}
     \begin{subfigure}[b]{0.9\textwidth}
         \centering
         \includegraphics[width=\textwidth]{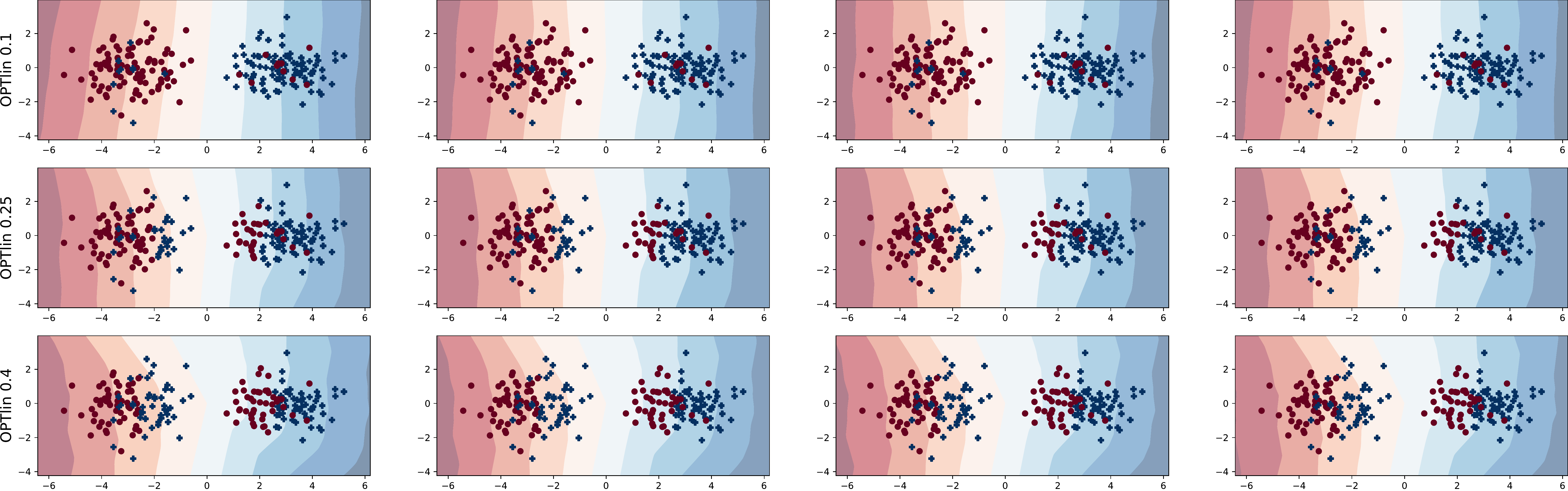}
         \caption{}
     \end{subfigure}
     \hfill
  
        \caption{(a) Test classification accuracy.  The differences with online SGD are essentially indistinguishable. (b)  Decision boundary when using 100 epochs multiple-pass SGD of batch size 32.  Columns correspond to different random initializations.  The decision boundary is more consistent across randomizations than the baseline online SGD algorithm.} \label{fig:batch}
\end{figure}

\begin{figure}[h]
     \centering
          \begin{subfigure}[b]{0.5\textwidth}
         \centering
         \includegraphics[width=\textwidth]{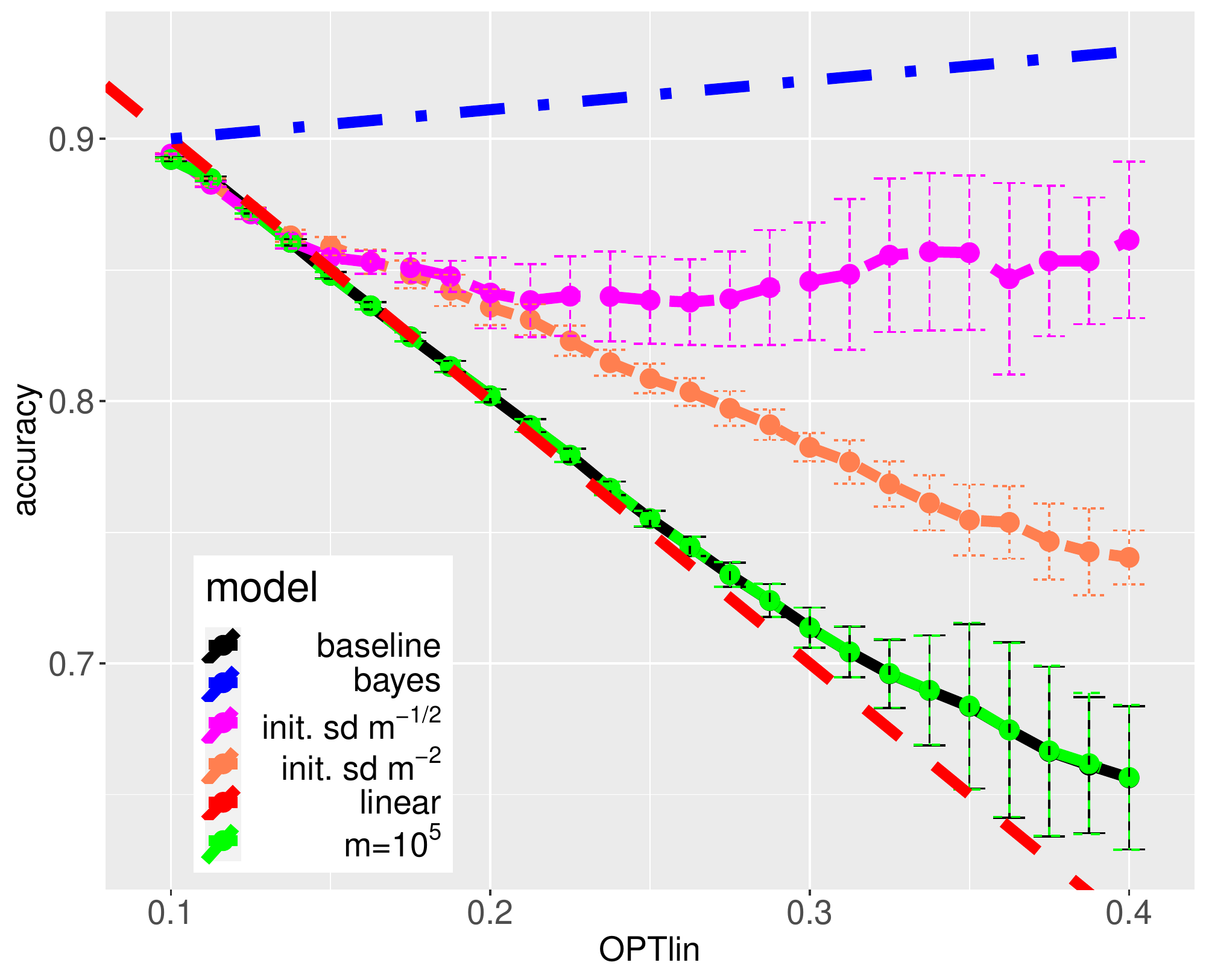}
         \caption{}
         \label{fig:secondandbias.results}
     \end{subfigure}
     \begin{subfigure}[b]{0.9\textwidth}
         \centering
         \includegraphics[width=\textwidth]{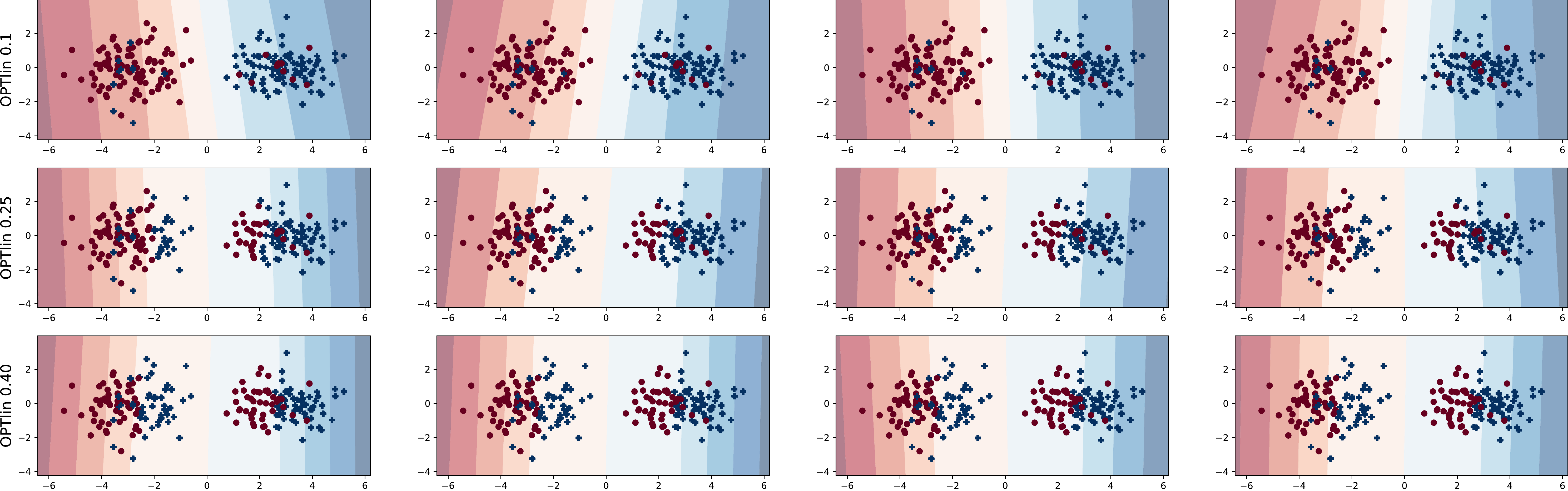}
         \caption{}
         \label{fig:secondandbias.decision.boundary.m-2}
     \end{subfigure}
     \begin{subfigure}[b]{0.9\textwidth}
         \centering
         \includegraphics[width=\textwidth]{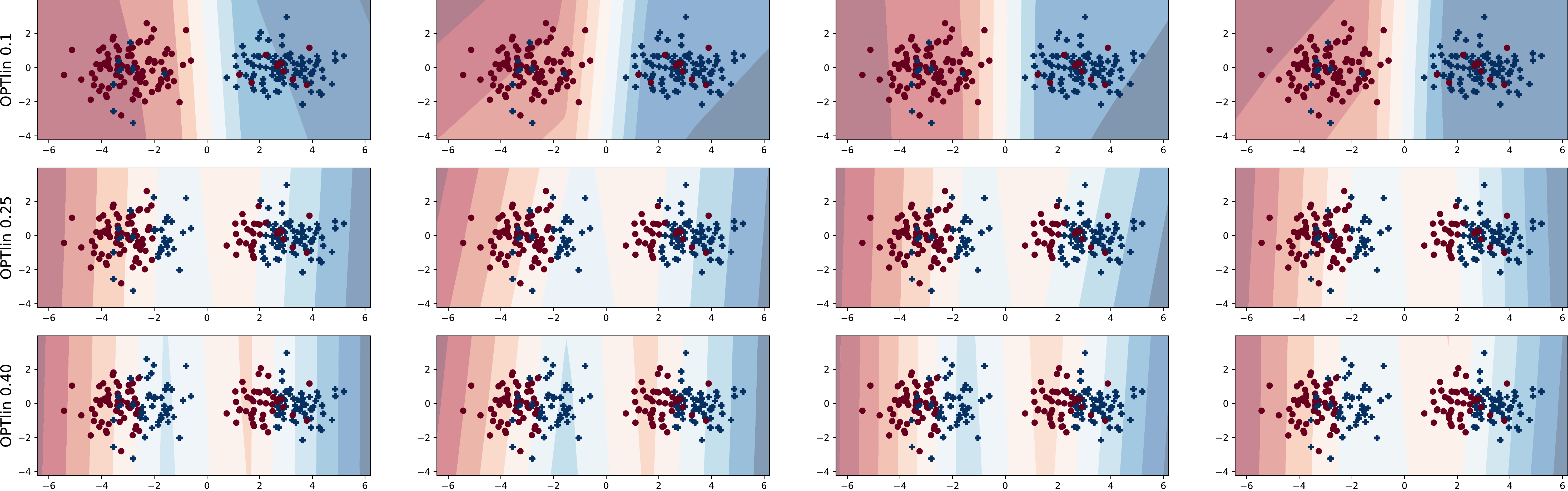}
         \caption{}
         \label{fig:secondandbias.decision.boundary.m-0.5}
     \end{subfigure}
        \caption{(a)  Test classification accuracy when introducing bias terms and trainable second layer weights (pink and coral dashed lines) as well as when increasing the width from $m=$ 1,000 to $m=$ 100,000 (green line).  The pink dashed line uses an initialization variance of $1/m$ while the coral dashed line uses an initialization variance of $1/m^4$.  Note that the performance of a neural network with width $m=$ 1,000 and width $m=$ 100,000 is imperceptible.  With trainable bias and second layer weights, the accuracy of the network varies significantly based on the initialization scheme.  Note that our result (Theorem \ref{thm:online.sgd.leaky.relu}) holds for an arbitrary initialization.   (b) Decision boundary when using trainable biases and second layer weights with an initialization variance of $1/m^4$.  The boundary is almost exactly linear.  (c) Same as (b) but using an initialization variance of $1/m$.  Here, the network can learn the appropriate nonlinear decision boundary.} \label{fig:secondandbias}
\end{figure}

\begin{figure}[ht!]
     \centering
     \begin{subfigure}[b]{0.68\textwidth}
         \centering
         \includegraphics[width=\textwidth]{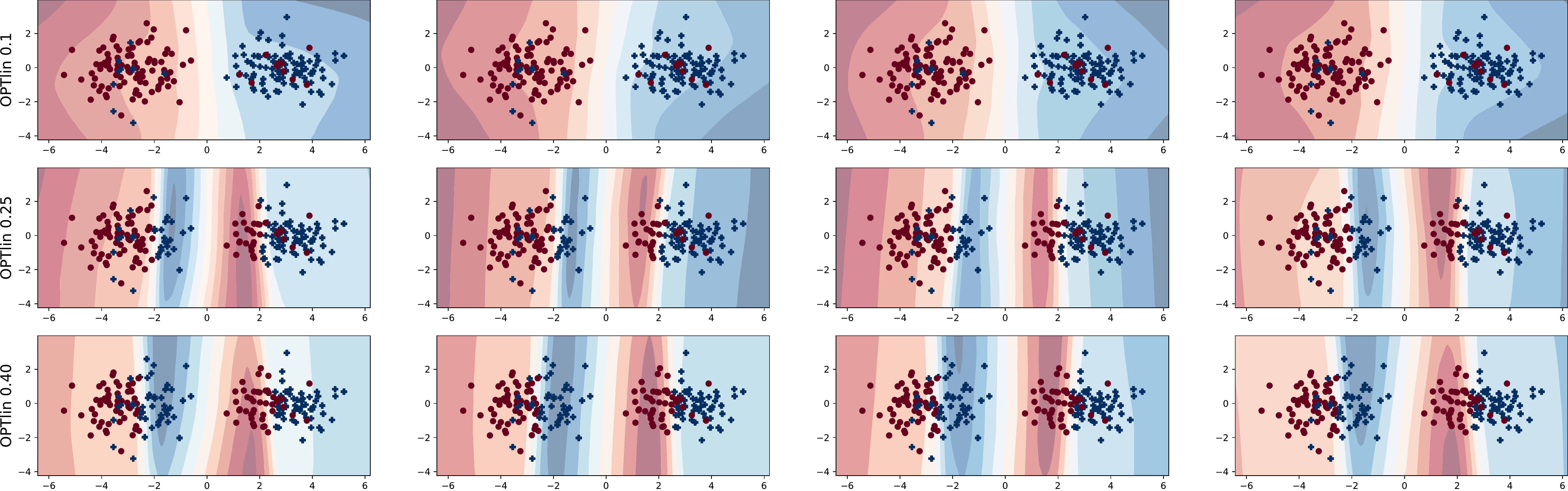}
         \caption{}
         \label{fig:4layer.decision.boundary}
     \end{subfigure}
     \hfill
     \begin{subfigure}[b]{0.3\textwidth}
         \centering
         \includegraphics[width=\textwidth]{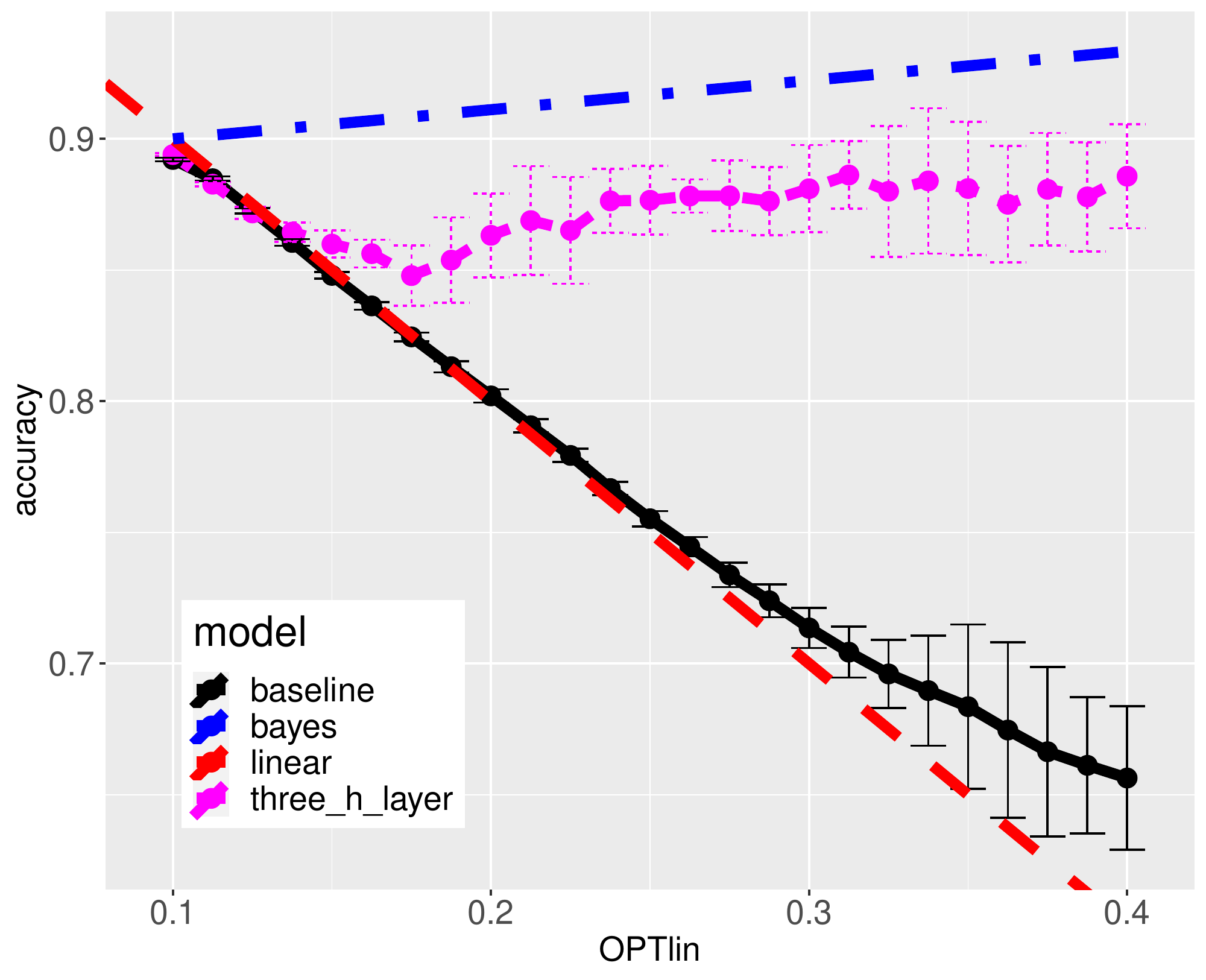}
         \caption{}
         \label{fig:4layer.results}
     \end{subfigure}
        \caption{(a)  Decision boundary for four layer residual network given in~\eqref{eq:4layer}. Columns correspond to different random initializations.  Compare with Figure \ref{fig:baseline.decision.boundary}.  With four layers, the network is able to appropriately partition the input space and generalize well.   (b) Test classification accuracy using the four layer network.  The four layer network accuracy is larger for $\optlin=0.4$ than it is for $\optlin=0.15$, a behavior closer to that of the Bayes classifier.} \label{fig:4layer}
\end{figure}

\begin{figure}[h!]
  \centering
  \includegraphics[width=0.25\textwidth]{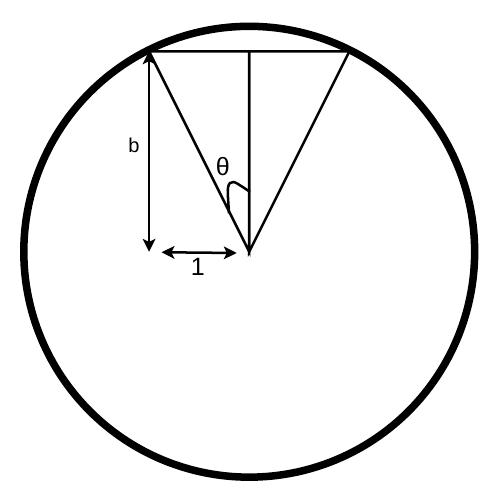}
  \caption{Calculation of the angle $2\theta$ for the distribution $\tilde \calD_{b}$ corresponding to the region $\{x_2 > b |x_1|\}$.}\label{fig:abs.distribution}
\end{figure}

\begin{figure}[ht!]
     \centering
     \begin{subfigure}[b]{0.68\textwidth}
         \centering
         \includegraphics[width=\textwidth]{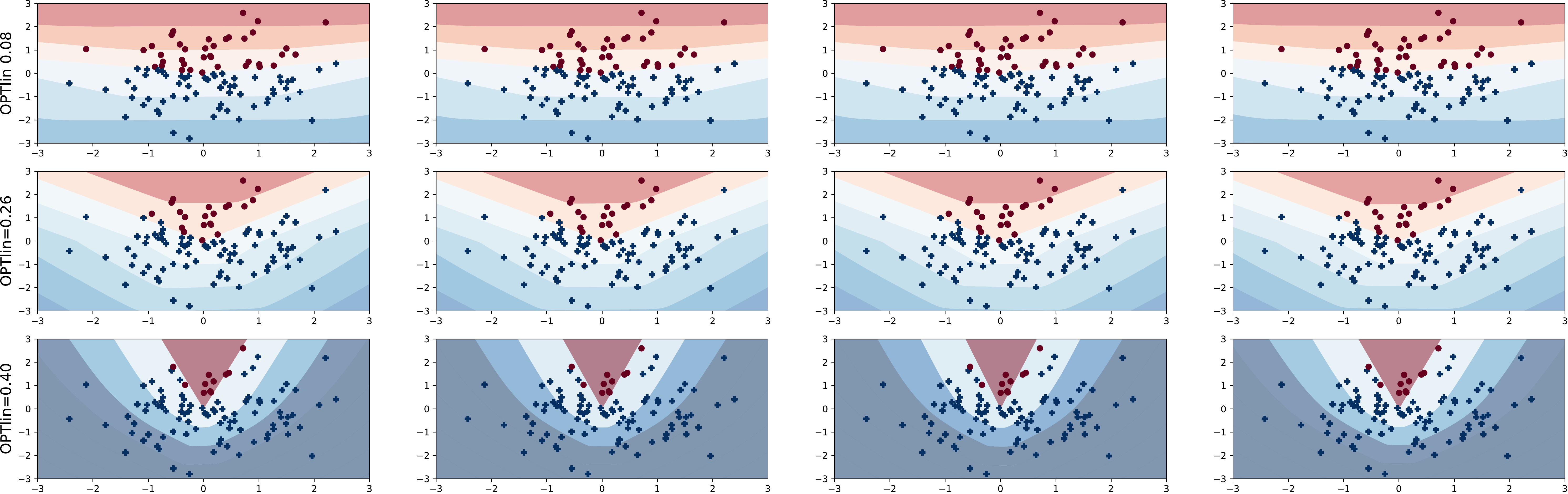}
         \label{fig:abs.decision}
     \end{subfigure}
     \hfill
     \begin{subfigure}[b]{0.3\textwidth}
         \centering
         \includegraphics[width=\textwidth]{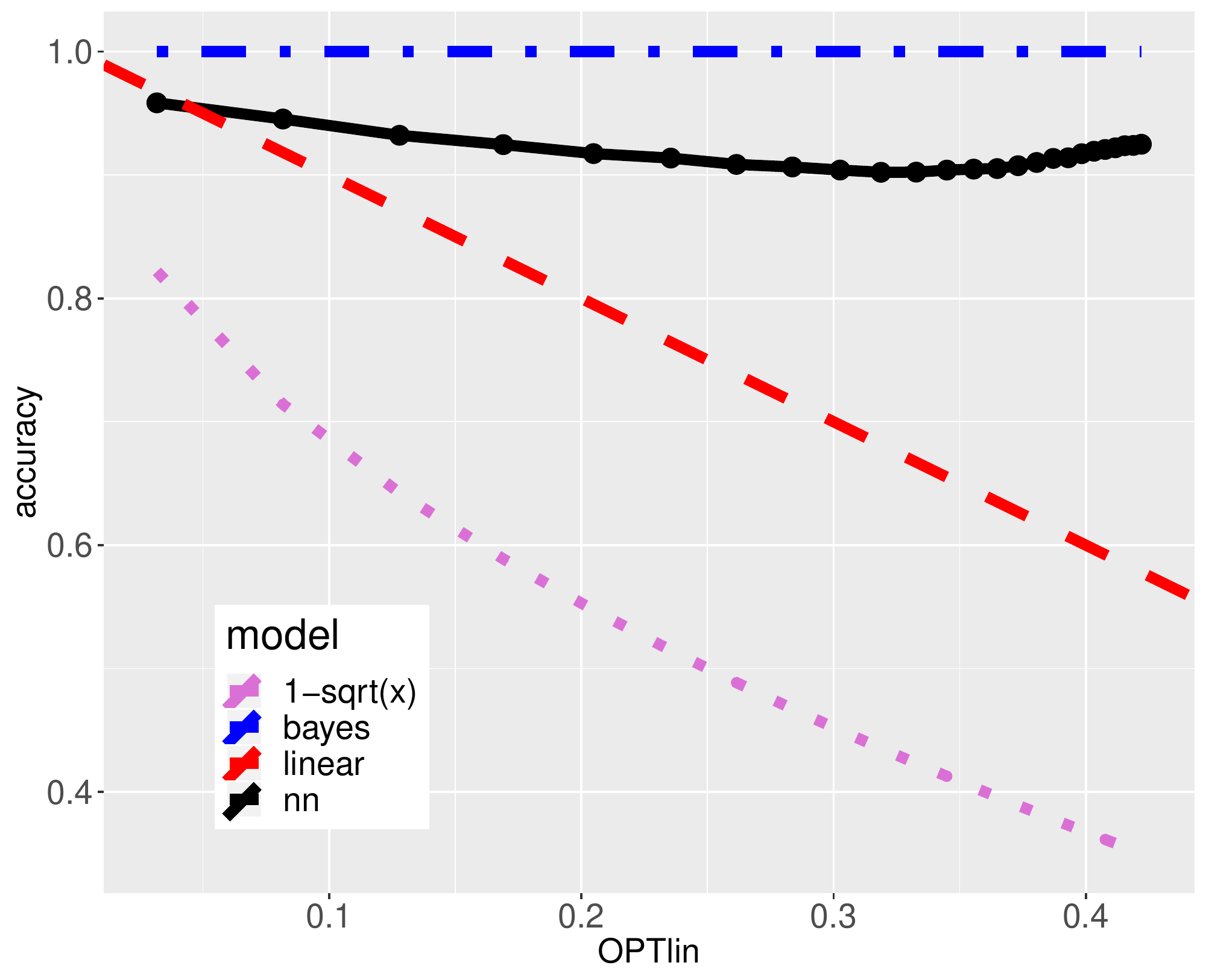}
         \label{fig:abs.results}
     \end{subfigure}
        \caption{(a)  Decision boundary for the same setup as the baseline neural network for data coming from $\tilde \calD_b$ for four random initializations (across columns) and for $\optlin \in  \{ 0.08, 0.26, 0.40\}$ (across rows).  Compare with Figure \ref{fig:baseline.decision.boundary}.  The decision boundaries are noticeably nonlinear.  (b) Test classification accuracy for data coming $\tilde \calD_{b}$.  Corollary \ref{corollary:anti.concentration} guarantees performance of at least $1-\Omega(\sqrt{\optlin})$, but the neural network performs significantly better due to the ability to produce a nonlinear decision boundary.  Note that the variance over ten initializations of the first layer weights are so small that the error bars are not visible.} \label{fig:abs}
\end{figure}


\clearpage
\bibliography{references}

\end{document}